\documentclass[11pt]{article}
\usepackage[letterpaper, margin=1in]{geometry}

\usepackage{tikz}
\usetikzlibrary{shapes.geometric, arrows, positioning, calc, fit, backgrounds}
\definecolor{expertgreen}{RGB}{220, 255, 220}
\definecolor{gateblue}{RGB}{220, 220, 255}
\definecolor{studentred}{RGB}{255, 220, 220}
\definecolor{routerorange}{RGB}{255, 240, 200}
\usepackage{pgfplots}
\usepgfplotslibrary{fillbetween}

\usepackage[utf8]{inputenc} 
\usepackage[T1]{fontenc}    
\usepackage{hyperref}       
\usepackage{url}            
\usepackage{booktabs}       
\usepackage{amsfonts}       
\usepackage{nicefrac}       
\usepackage{microtype}      
\usepackage{xcolor}         
\usepackage{graphicx}

\usepackage{algorithm}
\usepackage{algorithmic}

\usepackage{amsmath}
\usepackage{amssymb}
\usepackage{mathtools}
\usepackage{amsthm}

\usepackage{times}
\usepackage[capitalize,noabbrev]{cleveref}
\usepackage{natbib}

\usepackage{dsfont}
\usepackage[mathscr]{euscript}
\usepackage{expl3}
\usepackage{colortbl}
\usepackage{thmtools}
\usepackage{thm-restate}
\usepackage{multirow}
\usepackage{wrapfig}
\usepackage{accents}

\usepackage{pifont}

\usepackage{MnSymbol}
\DeclareMathAlphabet\mathbb{U}{msb}{m}{n}
\usepackage{xpatch}

\def\Nset{\mathbb{N}}
\def\Rset{\mathbb{R}}

\let\Pr\undefined

\DeclareMathOperator*{\Pr}{\mathbb{P}}

\DeclareMathOperator*{\E}{\mathbb E}

\DeclareMathOperator*{\argmin}{argmin}

\DeclareMathOperator{\sign}{sign}

\DeclarePairedDelimiter{\abs}{\lvert}{\rvert} 
\DeclarePairedDelimiter{\bracket}{[}{]}
\DeclarePairedDelimiter{\curl}{\{}{\}}
\DeclarePairedDelimiter{\paren}{(}{)}

\ExplSyntaxOn
\tl_const:Nn \c_my_uc_alphabet_tl { ABCDEFGHIJKLMNOPQRSTUVWXYZ }
\tl_const:Nn \c_my_full_alphabet_tl { ABCDEFGHIJKLMNOPQRSTUVWXYZ
  abcdefghijklmnopqrstuvwxyz }

\tl_map_inline:Nn \c_my_uc_alphabet_tl
 { \cs_gset:cpn { c#1 } { \mathcal{#1} } }

\tl_map_inline:Nn \c_my_uc_alphabet_tl
 { \cs_gset:cpn { s#1 } { \mathscr{#1} } }

\tl_map_inline:Nn \c_my_full_alphabet_tl
 {
  \cs_gset:cpn { b#1 } { \mathbf{#1} }
  \cs_gset:cpn { sf#1 } { \mathsf{#1} }
 }
\ExplSyntaxOff

\newcommand{\Rad}{\mathfrak R}

\newcommand{\balpha}{{\boldsymbol \alpha}}
\newcommand{\bbeta}{{\boldsymbol \beta}}
\newcommand{\bmu}{{\boldsymbol \mu}}
\newcommand{\bgamma}{{\boldsymbol \gamma}}

\newcommand{\hh}{{\sf h}}

\newcommand{\yy}{{\sf y}}

\newcommand{\num}{l}
\newcommand{\tp}{\textsf{TP}}

\newcommand{\fp}{\textsf{FP}}
\newcommand{\fn}{\textsf{FN}}
\newcommand{\sur}{\wt \sfL}
\newcommand{\ch}{\mathsf c_{\hh}}
\newcommand{\cy}{\mathsf c_{\yy}}
\newcommand{\cyy}{\mathsf c_{y'}}
\newcommand{\s}{{\sf s}}

\newcommand{\sh}{\mathsf s_{\hh}}
\newcommand{\sy}{\mathsf s_{\yy}}
\newcommand{\syy}{\mathsf s_{y'}}

\newcommand{\MMO}{\textsc{mmo}}

\newcommand{\h}{\widehat}
\newcommand{\ov}{\overline}
\newcommand{\uv}{\underline}
\newcommand{\wt}{\widetilde}
\newcommand{\e}{\epsilon}

\newcommand{\ul}{\uv \ell}
\newcommand{\ol}{\ov \ell}
\newcommand{\loss}{{\mathsf L}}

\usepackage{tikz}
\usetikzlibrary{shapes.geometric, arrows.meta, positioning, fit, calc}

\newcommand{\algorithmicreturn}{\textbf{return}}
\newcommand{\RETURN}{\STATE \algorithmicreturn\ } 

\hypersetup{
  breaklinks   = true, 
  colorlinks   = true, 
  urlcolor     = blue, 
  linkcolor    = blue, 
  citecolor    = blue 
}

\usepackage[toc, page, header]{appendix}
\declaretheorem{theorem}
\newtheorem{lemma}[theorem]{Lemma}

\newtheorem{corollary}[theorem]{Corollary}

\renewcommand{\algorithmicrequire}{\textbf{Input:}}
\renewcommand{\algorithmicensure}{\textbf{Output:}}

\title{Principled Algorithms for Optimizing Generalized\\ Metrics in Multi-Label Learning}
\author{
  Mehryar Mohri\\
  Google Research \& CIMS\\
  New York, NY 10011\\
  \texttt{mohri@google.com}
  \and
  Yutao Zhong\\
  Google Research\\ 
  New York, NY 10011\\
  \texttt{yutaozhong@google.com}
}
\date{}

\begin{document}

\maketitle

\addtocontents{toc}{\protect\setcounter{tocdepth}{0}}

\begin{abstract}
  Many real-world classification tasks require predicting multiple labels per instance, necessitating the optimization of complex evaluation metrics such as the $F$-measure and Jaccard index. While the Empirical Utility Maximization (EUM) framework is natural for these population-level metrics, existing theoretical results are largely limited to asymptotic Bayes-consistency. In this paper, we develop principled learning algorithms for optimizing a broad class of generalized metrics within the EUM framework, grounded in the stronger notion of $\sH$-consistency. Our key contribution is the design of novel surrogate loss functions for multi-label learning that admit provable $\sH$-consistency bounds, enabling optimization
  with non-asymptotic guarantees tailored to the hypothesis class and
  finite samples. Crucially, we prove these combinatorially formulated surrogates decompose exactly, operating in strictly $\mathcal{O}(\num)$ time without approximations. Building on this foundation, we introduce \MMO\ (\emph{Multi-Label Metric Optimization}), a new family of algorithms for optimizing generalized linear-fractional metrics. We validate our approach through extensive experiments, demonstrating robust scalability and superior performance over state-of-the-art continuous baselines on large-scale datasets (\texttt{MS-COCO}, \texttt{Reuters-21578}) in high-sparsity, deep learning regimes. Our results offer both theoretical rigor
  and practical effectiveness for general multi-label metric
  optimization.
\end{abstract}

\section{Introduction}

Multi-label classification, predicting multiple labels for a single instance, is a central task in modern machine learning, underpinning applications from image tagging to text categorization \citep{mccallum1999multi}. Unlike standard multi-class classification, evaluating performance in this setting often requires complex, non-decomposable metrics. Such metrics include averaged accuracy \citep{gao2011consistency}, averaged precision \citep{menon2019multilabel}, the $F$-measure and its averaged variants \citep{ye2012optimizing,dembczynski2011exact,dembczynski2013optimizing,waegeman2014bayes,zhang2020convex}, partial ranking loss \citep{gao2011consistency,dembczynski2012consistent}, the Jaccard measure \citep{sokolova2009systematic}, and others such as one-error, subset accuracy, coverage, macro/micro/instance-$F_1$ and AUC \citep{wu2017unified}.

Evaluating and optimizing these metrics requires a choice between two fundamental theoretical frameworks \citep{ye2012optimizing,koyejo2015consistent}. The first is \emph{Decision Theoretic Analysis (DTA)}, which evaluates performance based on the expectation of instance-level losses. This framework has received extensive attention, with recent work establishing strong consistency guarantees \citep{dembczynski2011exact,waegeman2014bayes,zhang2020convex,mao2024multi}. The second is \emph{Empirical Utility Maximization (EUM)},\footnote{Also referred to as the \emph{Population Utility (PU)} framework in some literature \citep{dembczynski2017consistency}.} which evaluates performance based on a utility function defined over the population-level confusion matrix (e.g., Micro-$F_1$, Macro-$F_1$). While DTA is well-studied, many crucial evaluation metrics are inherently global and best modeled under EUM. Despite this, consistency analysis within the EUM framework remains significantly underdeveloped. To our knowledge, \citet{koyejo2015consistent} remains the primary work offering a general consistency framework for EUM. However, their analysis establishes \emph{Bayes-consistency} \citep{steinwart2007compare}, an asymptotic property that guarantees optimality only over the class of all measurable functions.

Bayes-consistency, while theoretically foundational, often fails to provide practical guidance for modern machine learning, where we optimize over restricted, high-capacity hypothesis classes (like Neural Networks) using finite samples. An algorithm that is Bayes-consistent may fail to find the best predictor within a specific hypothesis set $\sH$. Furthermore, reliance on the structural properties of the Bayes-optimal predictor (e.g., thresholding) can be suboptimal when the hypothesis class does not contain the Bayes predictor.
To address these limitations, recent theoretical work has focused on the stronger notion of \emph{$\sH$-consistency} \citep{awasthi2022Hconsistency}. Unlike Bayes-consistency, $\sH$-consistency bounds provide non-asymptotic guarantees explicitly tailored to the hypothesis set $\sH$ used in practice. These bounds directly upper-bound the target estimation error by the surrogate estimation error, ensuring that minimizing the surrogate within $\sH$ effectively optimizes the target metric. While \citet{mao2024multi} addressed this for DTA, no such $\sH$-consistency bounds currently exist for the broader class of generalized metrics under the EUM framework.

\textbf{Our Contributions.} This paper bridges this gap by developing the first EUM-consistent framework grounded in $\sH$-consistency. Specifically:
\begin{enumerate}
\item In Section~\ref{sec:formulation}, we reformulate the optimization of micro-, macro-, and instance-averaged linear-fractional metrics as a cost-sensitive learning problem.
\item We design novel surrogate losses (Section~\ref{sec:surrogate}) that admit provable $\sH$-consistency guarantees. Unlike prior binary approaches, our multi-label analysis (Section~\ref{sec:guarantees}) gracefully handles exponential label dependencies ($2^\num$ combinations). Crucially, we prove that these combinatorial surrogates factorize exactly in strictly $\mathcal{O}(\num)$ time (Section~\ref{sec:efficient}), matching the computational efficiency of standard binary cross-entropy.
\item In Section~\ref{sec:algo}, we introduce \MMO\ (\emph{Multi-Label Metric Optimization}), a practical algorithmic framework for directly optimizing generalized fractional metrics.
\item We validate our approach extensively in Section~\ref{sec:experiments}, showing \MMO\ outperforms both classical EUM plug-in rules \citep{koyejo2015consistent} and state-of-the-art modern continuous losses \citep{benedict2022sigmoidf1,ridnik2021asymmetric} on standard benchmarks and large-scale, highly-sparse datasets using deep neural networks.
\end{enumerate}
We provide a more detailed discussion of related work in Appendix~\ref{app:related-work}.

\textbf{Novelty of Our Techniques.} 
While grounded in the principles of \citet{mao2025principled}, the transition to multi-label EUM introduces combinatorial complexity ($2^\num$ vs.\ 2 classes) that renders binary proof techniques inapplicable. Unlike the binary setting where surrogates rely on simple scalar structures, our framework must natively model intra-label dependencies inherent in metrics like Instance-Jaccard. Consequently, we introduce a distinct family of `comp-sum' surrogates and a novel multidimensional proof strategy for $\sH$-consistency that upper-bounds conditional regret over the entire $2^\num$-dimensional label vector space without suffering an exponential bound blow-up.

\section{Preliminaries}
\label{sec:pre}

We consider the setting of multi-label learning. Let $\sX$ denote the
input space and $\sY = \curl*{+1, -1}^\num$ represent the set of all
possible labels combinations, where $\num$ is a finite integer. For
instance, in image tagging tasks, $\sX$ could be a collection of
images, and $\sY$ could be a set of $\num$ predefined tags that can be
associated with each image. The total number of possible label
combinations is $n = \abs*{\sY} = 2^{\num}$. For any instance $x \in
\sX$ and its corresponding label vector $y = \paren*{y_1, \ldots,
  y_{\num}} \in \sY$, if $y_k = +1$, we say that that the $k$-th label
is relevant to $x$; otherwise, it is considered not relevant. Let
$[\num]$ denote the set $\curl*{1, \ldots, \num}$.

Let $\sH_{\rm{all}}$ be the family of all measurable functions $h
\colon \sX \times [\num] \to \Rset$. For notational convenience, we
represent the scoring vector produced by a hypothesis $h$ for an input
$x$ as $h(x) = \paren*{h(x, 1), \ldots, h(x, \num)}$. We define the
sign function as $\sign \colon t \mapsto 1_{t \geq 0} - 1_{t <
  0}$. The predicted label vector for an input $x \in \sX$ is denoted
by $\hh(x) \coloneqq \bracket*{\hh_1(x), \ldots, \hh_\num(x)} \in
\sY$, where $\hh_k(x) = \sign(h(x, k))$ for any $k \in
   [\num]$.

Given a loss function $\ell \colon \sH_{\rm{all}} \times \sX \times
\sY \to \Rset$, the \emph{generalization error} of a hypothesis $h$
under a data distribution $\sD$ is defined as $\sE_{\ell}(h) =
\E_{(x,y)\sim \sD} \bracket*{\ell(h, x, y)}$. The \emph{best-in-class
generalization error} for a hypothesis set $\sH \subseteq \sH_{\rm
  all}$ is $\sE_{\ell}^*(\sH) = \inf_{h \in \sH} \sE_{\ell}(h)$.
The \emph{excess error} of a hypothesis $h$ with respect to $\sH_{\rm
  all}$, given by $\sE_{\ell}(h) -
\sE_{\ell}^*\paren{\sH_{\rm{all}}}$, can be decomposed into the sum of
its \emph{estimation error}, $\sE_{\ell}(h) - \sE_{\ell}^*(\sH)$, and
the \emph{approximation error} of the class $\sH$, $\sE_{\ell}^*(\sH)
- \sE_{\ell}^*\paren{\sH_{\rm{all}}}$. For a given sample $S =
\paren*{(x^1, y^1), \ldots, (x^m, y^m)}$ of $m$ instances, the
\emph{empirical error} of a hypothesis $h$ is defined as $\h
\sE_{\ell, S}(h) = \frac{1}{m} \sum_{i = 1}^m \ell(h, x^i, y^i)$.

\textbf{Bayes-Consistency}. When considering a surrogate loss function
$\ell_1$ and a target loss function $\ell_2$, a fundamental property
relating them is \emph{Bayes-consistency}
\citep{Zhang2003,zhang2004statistical,bartlett2006convexity,
  tewari2007consistency}.  A loss function $\ell_1$ is
\emph{Bayes-consistent} with respect to a loss function $\ell_2$ if,
for all distributions and all sequences of hypotheses $\{h_n\}_{n\in
  \Nset}\subset \sH_{\rm all}$, the condition $\curl*{\sE_{\ell_1}(h_n) -
\sE_{\ell_1}^*\paren*{\sH_{\rm{all}}} \to 0}$ as $n \to +\infty$
implies $\curl*{\sE_{\ell_2}(h_n)-\sE_{\ell_2}^*\paren*{\sH_{\rm{all}}}
\to 0}$.

While Bayes-consistency is a natural and desirable property, it is
purely asymptotic and applies only to the unrestricted class of all
measurable functions $\sH_{\rm{all}}$. As such, it offers no guidance
on convergence rates or on performance under practical constraints,
where learning is restricted to specific hypothesis classes.

\textbf{$\sH$-Consistency Bound}. A more informative and practical
guarantee in this setting is an $\sH$-consistency bound
\citep{awasthi2022Hconsistency,mao2023cross}, which quantifies how
well a surrogate loss $\ell_1$ approximates a target loss $\ell_2$
over a specific hypothesis class $\sH$. Formally, $\ell_1$ admits an
$\sH$-consistency bound with respect to $\ell_2$ if there exists a
non-decreasing, concave function $\Gamma \colon \Rset_+ \to \Rset_+$
with $\Gamma(0) = 0$, such that for all $h \in \sH$ and any
distribution:
$
\sE_{\ell_2}(h) - \sE_{\ell_2}^*(\sH) + \sM_{\ell_2}(\sH)
\leq \Gamma \paren{\sE_{\ell_1}(h)
- \sE_{\ell_1}^*(\sH) + \sM_{\ell_1}(\sH)},$
where $\sM_{\ell}(\sH) := \sE^*_\ell(\sH) - \E\left[\inf_{h \in \sH}
  \E[\ell(h, x, y) \mid x]\right]$ is the minimizability gap. The
latter term reflects the gap between the best-in-class error and the
best achievable loss when selecting an optimal $h$ for each $x$. For
rich classes like $\sH = \sH_{\text{all}}$, this gap vanishes,
reducing the bound to $\sE_{\ell_2}(h) - \sE_{\ell_2}^*(\sH) \leq
\Gamma(\sE_{\ell_1}(h) - \sE_{\ell_1}^*(\sH))$, which implies
Bayes-consistency. In general, the gap is non-negative, bounded by the
approximation error of $\sH$, and often significantly smaller. Thus,
$\sH$-consistency bounds yield sharper, non-asymptotic, and
class-specific guarantees of consistency.

\section{Problem formulation}
\label{sec:formulation}

Our goal is to design a principled learning algorithm that can
optimize a broad family of generalized performance metrics for
multi-label learning. We begin by formally defining this family of
metrics and then reformulate the problem using an alternative
loss function to enable more tractable optimization.

\subsection{Generalized Metrics} 
\label{sec:generalized-metrics}

Given a sample $S = \paren*{(x^1, y^1), \ldots, (x^m, y^m)}$, we
consider three types of \emph{empirical generalized metrics} for
multi-label learning: \emph{micro-averaged}, \emph{macro-averaged},
and \emph{instance-averaged} metrics.  For any $\bmu = [\mu_1^1,
  \mu_1^2, \mu_1^3, \mu_1^4, \ldots, \mu_\num^1, \mu_\num^2,
  \mu_\num^3, \mu_\num^4] \in \Rset^{4 \num}$, label index $k \in
[\num]$, and example $(x, y) \in \sX \times \sY$, we define the shorthand
$\ell_{\bmu, k} (x, y) = \mu_k^1 \hh_k(x) y_k + \mu_k^2 y_k + \mu_k^3
\hh_k(x) + \mu_k^4$.  Then, for any hypothesis $h \in \sH_{\rm{all}}$,
the empirical metrics are defined as follows:
\begin{align}
\label{eq:target-empirical}
\h \cL_{S, \rm{micro}}(h) &= \frac{\sum_{i = 1}^m \sum_{k = 1}^{\num} \ell_{\balpha, k}(x^i, y^i)} {\sum_{i = 1}^m \sum_{k = 1}^{\num} \ell_{\bbeta, k}(x^i, y^i)}, \nonumber \\
\h \cL_{S, \rm{macro}}(h) &= \frac{1}{\num} \sum_{k = 1}^{\num} \frac{\sum_{i = 1}^m \ell_{\balpha, k}(x^i, y^i)} {\sum_{i = 1}^m \ell_{\bbeta, k}(x^i, y^i)}, \nonumber \\
\h \cL_{S, \rm{instance}}(h) &= \frac{1}{m} \sum_{i = 1}^{m} \frac{ \sum_{k = 1}^{\num} \ell_{\balpha, k}(x^i, y^i)} { \sum_{k = 1}^{\num} \ell_{\bbeta, k}(x^i, y^i)},
\end{align}
with $\balpha, \bbeta \in \Rset^{4 \num}$.
To maintain notational simplicity, the dependence on $(\balpha,
\bbeta)$ is assumed to be clear from the context and will be omitted
unless essential. We may also use $\h \cL_{S}$ without the subscript
indicating the averaging method when it is evident from the
context. The construction of these metrics involves different
averaging strategies:
\textbf{Micro-averaging}: Aggregates the contributions from all
  instance-label pairs before forming the ratio.
\textbf{Macro-averaging}: Computes a ratio for each label by
  averaging over instances, and then averages these label-specific
  ratios.
 \textbf{Instance-averaging}: Computes a ratio for each instance
  by averaging over labels, and then averages these instance-specific
  ratios.
The fundamental statistics for multi-label classification, true
positives (TP), false positives (FP), true negatives (TN), and false
negatives (FN) \citep{koyejo2015consistent}, can be expressed using
$\hh_k(x^i) y_k^i$, $ y_k^i$, and $\hh_k(x^i)$. Specifically, by
defining $\ov \hh_k(x^i) = \frac{\hh_k(x^i) + 1}{2} \in \curl*{0, 1}$
and $\ov y_k^i = \frac{y_k^i + 1}{2} \in \curl*{0, 1}$:
\begin{align*}
  \text{TP}_{k, i} &= \ov \hh_k(x^i) \ov y_k^i &
  \text{FP}_{k, i} &= \ov \hh_k(x^i) \bracket*{1 - \ov y_k^i}\\
  \text{TN}_{k, i} &= \bracket*{1- \ov \hh_k(x)} \bracket*{1 - \ov y_k^i} &
  \text{FN}_{k, i} &= \bracket*{1 - \ov \hh_k(x^i)} \ov y_k^i .
\end{align*}
Since each of these (TP, FP, TN, FN) can be expressed as a linear
combination of $\hh_k(x^i) y_k^i$, $ y_k^i$, $\hh_k(x^i)$, and a
constant, any metric that is a ratio of two linear combinations of
micro-averaged, macro-averaged, or instance-averaged TP, FN, TN, and
FP counts (as considered in \citep{koyejo2015consistent}) can be
expressed in the form of \eqref{eq:target-empirical}. This general
formulation encompasses numerous widely used metrics, such as averaged
accuracy (or $1 - $ Hamming loss), averaged precision, averaged
$F$-measure, and the averaged Jaccard measure, among others.

Given a hypothesis set $\sH$ and a distribution $\sD$ over $\sX \times
\sY$, our goal is to identify a hypothesis $h \in \sH$ that maximizes
the \emph{population generalized metrics}.  These are derived from
their empirical counterparts in \eqref{eq:target-empirical} by
replacing individual terms like $\hh_k(x^i) y_k^i$, $ y_k^i$, and
$\hh_k(x^i)$ with their respective expectations: $\E_{(x, y)}
\bracket*{\hh_k(x) y_k}$, $\E_{(x, y)} \bracket*{y_k}$, and $ \E_{(x,
  y)} \bracket*{\hh_k(x)}$.
Under this transition to expectations, the population versions of both
micro-averaged and instance-averaged metrics, $\cL(h)$, and the
population version of the macro-averaged metric, $\wt \cL(h)$ can be
expressed as follows:
\begin{align}
\label{eq:target-population}
\cL(h) = \frac{\sum_{k = 1}^{\num}
  \E_{(x, y) \sim \sD} \bracket*{\ell_{\balpha, k}(x, y)} }
   {\sum_{k = 1}^{\num} \E_{(x, y) \sim \sD} \bracket*{\ell_{\bbeta, k}(x, y)}} \qquad
   \wt \cL(h) = \frac{1}{l} \sum_{k = 1}^{\num} \frac{\E_{(x, y) \sim \sD}
\bracket*{\ell_{\balpha, k}(x, y)} }
   {\E_{(x, y) \sim \sD} \bracket*{\ell_{\bbeta, k}(x, y)}}.
\end{align}
We denote the optimal population generalized metrics within class
$\sH$ as $\cL^*(\sH) = \sup_{h \in \sH} \cL(h)$ and $\wt \cL^*(\sH) =
\sup_{h \in \sH} \wt \cL(h)$. Note that if the hypothesis set $\sH$
allows each scoring function $h(x, k)$ to be optimized independently
for each label $k$ (a common assumption in multi-label learning, see
e.g., \citep{koyejo2015consistent}), then optimizing $\wt \cL$ as
defined in \eqref{eq:target-population} is equivalent to independently
optimizing the per-label ratio: $\frac{ \E_{(x, y) \sim \sD}
  \bracket*{ \ell_{\balpha, k}(x, y)} } {\E_{(x, y) \sim \sD}
  \bracket*{\ell_{\bbeta, k}(x, y)}}$ for each $k \in [\num]$. Each
such per-label optimization can be viewed as a special case of
optimizing $\cL$ in \eqref{eq:target-population} by setting $\num = 1$
(that is, considering only a single label). Therefore, in what
follows, our analysis will focus on optimizing $\cL$.
    
\subsection{Equivalent Problem Formulation}
\label{sec:equivalent-problem}

The generalized metric $\cL(h)$ defined in \eqref{eq:target-population} is
a ratio of two expected values. This fractional structure differs from
the more conventional single-expectation loss functions typically used
when deriving and analyzing surrogate losses in machine learning. This
section introduces an equivalent reformulation for the problem of
maximizing $\cL(h)$, providing a more amenable framework for designing
surrogate loss functions.

For any $\bmu \in \Rset^{4 \num}$, we define the auxiliary function
$\ell_{\balpha} = \sum_{k = 1}^{\num} \ell_{\balpha, k}$. $\cL(h)$ can
then be expressed as follows: $\cL(h) = \frac{\E_{(x, y) \sim \sD}
  \bracket*{\ell_{\balpha}(h, x, y)}}{\E_{(x, y) \sim \sD}
  \bracket*{\ell_{\bbeta}(h, x, y)}}$, which is a fractional program
involving the expected values of $\ell_{\balpha}$ and
$\ell_{\bbeta}$. To maximize $\cL(h)$, we reformulate this into an
equivalent optimization problem that will simplify further analysis
and the design of surrogates.
For any scalar $\lambda$ (and any given $\balpha$ and $\bbeta$), we
define the loss function $\ell^{\lambda}$ as follows:
\begin{equation}
\label{eq:target-equiv}
\forall (h, x, y), \, \ell^{\lambda}(h, x, y)
=  \lambda \ell_{\bbeta}(h, x, y) - \ell_{\balpha}(h, x, y).
\end{equation}
Let $\h \sE_{\ell^\lambda, S}$ denote the average empirical loss of
$\ell^\lambda$ over a sample $S$. Without loss of generality, we
assume throughout that the denominator $\E_{(x, y) \sim \sD}
\bracket*{\ell_{\bbeta}(h, x, y)}$ is positive for all $h \in \sH$. If
this condition is not met, it can be satisfied by simultaneously
negating both $\ell_{\balpha}$ and $\ell_{\bbeta}$. For convenience,
we further define $\ul_\bbeta = \inf_{h \in \sH} \E_{(x, y) \sim
  \sD}[\ell_\bbeta(h, x, y)]$ and assume $\ul_\bbeta > 0$. Similarly,
let $\ol_\bbeta = \sup_{h \in \sH} \E_{(x, y) \sim \sD}[\ell_\bbeta(h,
  x, y)]$, assuming $\ol_\bbeta < +\infty$.  The following theorem,
based on principles from fractional programming, establishes the
equivalence between maximizing $\cL(h)$ and minimizing the expected
value of $\ell^{\lambda^*}$, where $\lambda^* = \cL^*(\sH)$ is the
optimal value of the original problem.

\begin{restatable}{theorem}{Equiv}
\label{thm:equiv}
The equality $\cL(h^*) = \cL^*(\sH)$ holds for some $h^* \in \sH$ if and
only if $h^*$ minimizes $\sE_{\ell^{\lambda^*\!\!}}(h)$ over $\sH$, and $\sE_{\ell^{\lambda^*\!\!}}(h^*) = \sE^*_{\ell^{\lambda^*\!\!}}(\sH) = 0$.
\end{restatable}
More generally, the following non-asymptotic equivalence provides a
quantitative relationship. 
\begin{restatable}{theorem}{EquivNon}
  \label{thm:equiv-non}
Fix $\eta \geq 0$ and $h \in \sH$. Then, 
$\paren*{\sE_{\ell^{\lambda^*}}(h) \leq \eta}$ holds iff $\cL^*(\sH) - \cL(h) \leq
\frac{\eta}{\E_{(x, y) \sim \sD} \bracket*{ \ell_{\bbeta}(h, x, y)
}}$.
\end{restatable}
The proofs of these theorems are presented in
Appendix~\ref{app:equiv}.
Theorems~\ref{thm:equiv} and \ref{thm:equiv-non} demonstrate that our
original problem of maximizing $\cL(h)$ can be reduced to minimizing
the expected loss $\sE_{\ell^{\lambda^*}}(h)$. However, direct
minimization of $\sE_{\ell^{\lambda^*}}(h)$ remains challenging
because the function $\ell^{\lambda^*\!\!}$, through its dependence on
$\hh_k(x) = \sign(h(x, k))$, is non-differentiable and non-continuous
with respect to the score values $h(x, k)$. In the subsequent section,
we will introduce a general family of surrogate loss
functions for $\ell^{\lambda^*\!\!}$ that benefit from strong
consistency guarantees and are better suited for
practical optimization algorithms.

\section{Cost-Sensitive Multi-Label Classification}
\label{sec:mll}

In this section, we first demonstrate that the loss function
$\ell^{\lambda^*\!\!}$, derived in the preceding section, can be
interpreted as a specific type of cost-sensitive multi-label target
loss. Building on this, we address the broader problem of
cost-sensitive multi-label classification by introducing a novel
family of surrogate loss functions tailored to this
framework. Finally, we present strong theoretical guarantees for these
surrogates, underscoring their suitability for developing efficient
algorithms to optimize both general cost-sensitive multi-label
objectives and, consequently, the original loss
$\ell^{\lambda^*\!\!}$.

\subsection{General Target Loss Function}
\label{sec:target}

Let $\bgamma = [\gamma_1^1, \gamma_1^2, \gamma_1^3, \gamma_1^4,...,
  \gamma_\num^1, \gamma_\num^2, \gamma_\num^3, \gamma_\num^4]$ be a
coefficient vector where $\gamma_k^j = \lambda^* \beta_k^j -
\alpha_k^j$, for $j \in \curl*{1, 2, 3, 4}$ and $k \in [l]$. The loss
function $\ell^{\lambda^*\!\!}(h, x, y)$ can then be written as:
\begin{align}
\label{eq:target-equiv-2}
\sum_{k = 1}^{\num} \paren*{\gamma_k^1 \hh_k(x) y_k
    + \gamma_k^2 y_k + \gamma_k^3 \hh_k(x) + \gamma_k^4} \eqqcolon \sfL_{\bgamma}(\hh(x), y), 
\end{align}
where $\sfL_{\bgamma} \colon \sY \times \sY \to \Rset$ is a
parameterized loss function. Since adding a constant does not affect
minimization, we may shift $\sfL_{\bgamma}$ by any constant $C_0 \geq
0$ to ensure non-negativity: $\sfL_{\bgamma}(\cdot, \cdot) + C_0 \geq
0$.
This formulation of $\sfL_{\bgamma}(\hh(x), y)$ naturally fits into
the general framework of \emph{cost-sensitive multi-label
classification}, where a cost function $\ov \sfL \colon \sY \times \sY
\to \Rset_+$ quantifies the penalty of predicting $\hh(x)$ when the
true label is $y$. The corresponding target loss function is defined
as:
\begin{equation} 
\label{eq:mll}
\forall (h, x, y), \quad \sfL(h, x, y) = \ov \sfL(\hh(x), y).
\end{equation}

\subsection{General Surrogate Losses}
\label{sec:surrogate}

The cost function $\ov \sfL$ in Equation~\eqref{eq:mll} can be very
general; it need not be symmetric (i.e., $\ov \sfL(y, y') = \ov
\sfL(y', y)$ may not hold) nor definite (i.e., $\ov \sfL(y, y')$ is
not necessarily zero if $y = y'$. This defines a general
cost-sensitive learning problem, extending the work of
\citet{mao2025principled} where the label space $\sY$ was binary. Our
multi-label setting (where $\sY = \curl*{+1, -1}^\num$) significantly
broadens this scope and includes the binary case as a special instance
when $\num = 1$.

Directly minimizing the general cost-sensitive loss $\sfL$ is often
intractable for most hypothesis sets, primarily due to its
non-continuous and non-differentiable nature (inherited from the
$\sign$ function in $\hh(x)$). Consequently, practical approaches
typically involve optimizing a surrogate loss function. Effective
surrogate losses are designed to be consistent (or, more strongly,
$\sH$-consistent) with the target loss
\citep{Zhang2003,bartlett2006convexity,awasthi2022Hconsistency,awasthi2022multi,mao2023cross,mao2025principled}
and are often formulated as margin-based losses
\citep{lin2004note}. \citet{mao2025principled} recently extended such
margin-based losses to the general cost-sensitive binary
classification setting.
 
Here, we propose a novel family of surrogate loss functions for the
general cost-sensitive multi-label loss $\sfL$. These surrogates
extend the comp-sum loss function framework \citep{mao2023cross} to
this more general setting:
\begin{align}
\label{eq:sur}
  \sfL_{\tau}(h, x, y)
  = \sum_{y' \in \sY} \, \bracket*{\ov S - \ov\sfL(y', y)}
    \Phi_{\tau} \paren*{\sum_{y'' \in \sY} e^{\sum_i \paren*{ y''_i - y'_i} h(x, i)}}.
\end{align}
where $y$ is the true label vector,
$\sfS = \sum_{y', y \in \sY} \ov\sfL(y', y)$, and for any
$\tau \geq 0$, the function $\Phi_{\tau}$ is defined for all
$u \geq 0$ by:
$\Phi_{\tau}(u) = \frac{1}{\tau} \paren*{u^{-\tau} - 1} 1_{\tau \geq
  0} + \log(u) 1_{\tau = 0}$.  In the special case where $\tau = 0$,
$\sfL_{0}$ becomes:
$\sum_{y' \in \sY} \, \bracket*{\sfS - \ov\sfL(y', y)} \cdot \log
\paren*{\sum_{y'' \in \sY} e^{\sum_{i = 1}^\num \paren*{ y''_i - y'_i}
    h(x, i)}}.  $ Note that when $\num = 1$ (the binary case), this
formulation recovers the cost-sensitive logistic loss proposed in
\citep{mao2025principled}.  Different choices for the parameter $\tau$
yield different types of cost-sensitive surrogate losses: $\tau = 0$
leads to a cost-sensitive \emph{logistic-type loss}
\citep{Verhulst1838,Verhulst1845,Berkson1944,Berkson1951};
$\tau \in (0, 1)$ gives a cost-sensitive \emph{generalized
  cross-entropy-type loss} \citep{zhang2018generalized}; and
$\tau = 1$ corresponds to a cost-sensitive \emph{mean absolute
  error-type loss} \citep{ghosh2017robust}.

\subsection{Theoretical Guarantees}
\label{sec:guarantees}

This section establishes strong theoretical guarantees for the
proposed surrogate loss $\sfL_{\tau}$. In particular, we derive
$\sH$-consistency bounds for $\sfL_{\tau}$ with respect to the general
cost-sensitive target loss $\sfL$ (defined in \eqref{eq:mll}),
for widely used hypothesis classes.

We demonstrate that these cost-sensitive comp-sum type losses admit
$\sH$-consistency bounds when the hypothesis set $\sH$ is
\emph{symmetric} and \emph{complete}.  A hypothesis set $\sH$ is
\emph{symmetric} if the set of all possible score vectors
$\curl*{(h(x,1), \ldots, h(x,\num))\colon h \in \sH }$ can be
generated by independently choosing $\num$ real-valued functions $f_1,
\ldots, f_{\num}$ from a common underlying family $\sF$ for each
component, that is, $\curl*{(f_1(x), \ldots, f_{\num}(x))\colon f_1,
  \ldots, f_{\num} \in \sF}$, for any $x \in \sX$.  A hypothesis set
$\sH$ is \emph{complete} if, for any input $x$ and any label index $k
\in [\num]$, the scores $h(x,k)$ can span $\Rset$; that is, $\curl*{h(x,k) : h \in \sH} = \Rset$.  As noted by \citet{awasthi2022multi} and
\citet{mao2023cross}, these assumptions are quite general and are
satisfied by common hypothesis sets in practice, including linear
hypotheses and multi-layer feed-forward neural networks, as well as
the set of all measurable functions.

\begin{restatable}[\textbf{$\sH$-consistency bound of
      $\sfL_{\tau}$}]
  {theorem}{BoundComp}
\label{Thm:bound_comp}
Assume that $\sH$ is symmetric and complete. Then, for any target
cost-sensitive loss $\sfL$, any hypothesis $h \in \sH$, and any
distribution, we have:
\begin{equation}
\label{eq:bound_comp}
\sE_{\loss}(h)-\sE_{\loss}^*(\sH)
\leq \Gamma\paren*{\sE_{\sfL_{\tau}}(h) -
  \sE_{\sfL_{\tau}}^*(\sH) + \sM_{\sfL_{\tau}}(\sH)} - \sM_{\loss}(\sH),
\end{equation}
where $\ov\sfL_{\rm{sum}} =\sum_{y', y \in \sY} \ov\sfL(y', y)$, and
$\Gamma(t)$ is defined as: $\Gamma(t) = 2 \sqrt{\ov \sfL_{\rm{sum}}
  n^{\tau} t} \, 1_{\tau \in [0, 1)} + \tau n^{\tau} t \, 1_{\tau \geq 1}$.
\end{restatable}
The analysis here departs significantly from the binary setting in
\citep{mao2025principled}. In the binary case, $\sH$-consistency
proofs often leverage the simplicity of the scalar score margin. In
the multi-label setting, the conditional regret involves a supremum
over a combinatorial structure $\sY$. Our proof of
Theorem~\ref{Thm:bound_comp} relies on a novel construction of a
comparator hypothesis $h_\mu$ (Lemma~\ref{lemma:delta_target} in Appendix~\ref{app:proof-bound-comp}) that
perturbs the predictive distribution on specific label configurations,
a technique without a direct analogue in the binary domain. The complete formal proof of Theorem~\ref{Thm:bound_comp} is provided in Appendix~\ref{app:bound_comp-sum}.
This proof strategy is novel and fundamentally different from the one
used in \citet{mao2023cross}, which applies only to the zero-one
loss. Their method operates on individual scores $h(x, y)$, whereas
ours treats the predictive distribution as a whole. This shift enables
a much simpler constrained optimization over $\mu$ that remains
tractable even for general cost-sensitive losses, something the
previous method cannot accommodate.

Moreover, by \citep[lemma~2.1]{mao2024universal}, the minimizability
gaps $\sM_{\sfL_{\tau}}(\sH)$ and $\sM_{\sfL}(\sH)$ vanish for the
family of all measurable functions ($\sH = \sH_{\rm{all}}$). Thus, the
$\sH$-consistency bounds in Theorem~\ref{Thm:bound_comp} directly
imply Bayes-consistency for these cost-sensitive comp-sum type losses.
\begin{corollary}
The cost-sensitive comp-sum type loss $\sfL_{\tau}$ is
Bayes-consistent with respect to the target cost-sensitive loss
$\sfL$.
\end{corollary}
Theorem~\ref{Thm:bound_comp} offers stronger, quantitative bounds than
Bayes-consistency, especially when minimizability gaps are zero. It
implies that if the estimation error of the surrogate
$\sE_{\sfL_{\tau}}(h)-\sE_{\sfL_{\tau}}^*(\sH)$ is reduced to $\e$,
then the estimation error of the target loss
$\sE_{\sfL}(h)-\sE_{\sfL}^*(\sH)$ is upper-bounded by $2\sqrt{\ov
  \sfL_{\rm{sum}}\e}$ for the cost-sensitive logistic-type loss
($\tau=0$), by $2\sqrt{\ov \sfL_{\rm{sum}} n^{\tau}\e}$ for the
cost-sensitive generalized cross-entropy-type loss ($\tau \in (0,1)$),
and by $n \e$ for the cost-sensitive mean absolute error-type loss
(when $\tau=1$).

\subsection{Computational Scalability: Exact \texorpdfstring{$\mathcal{O}(\num)$}{O(l)} Factorization}
\label{sec:efficient}
A naive implementation of the surrogate loss $\sfL_{\tau}$ (Eq.~\eqref{eq:sur}) appears to require evaluating summations over the entire combinatorial label space $\sY$, which scales as $\mathcal{O}(2^\num)$. This would render the method completely intractable for datasets with large numbers of labels. However, because our generalized target metrics decompose linearly over individual labels (as defined in Eq.~\eqref{eq:target-equiv-2}), the cost term is purely additive across labels. We formally prove that this additive structure intertwines with the exponential comp-sum surrogate such that the $2^\num$ combinatorial summations perfectly factorize into a product of $\num$ independent terms. 

For instance, when $\tau = 0$ (the logistic-type surrogate used in our deep learning experiments), the logarithm converts the inner exponential product into a sum of $\num$ independent binary losses. The outer summation over $\sY$ evaluates matching and cross-terms that can be collapsed precisely using precomputed 1D marginal arrays. We rigorously prove that computing the exact surrogate loss $\sfL_{\tau}$ requires only $\mathcal{O}(\num)$ operations, scaling linearly with the number of labels and exactly matching the computational efficiency of standard Binary Cross-Entropy (BCE) with zero approximations. The complete algebraic derivation of this exact $\mathcal{O}(\num)$ factorization for all $\tau \ge 0$ is provided in Appendix~\ref{app:factorization}.

\section{Algorithms for Optimizing Generalized Metrics}
\label{sec:algo}

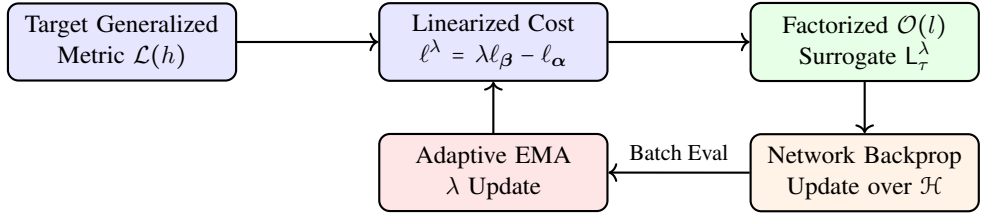
\begin{figure}[t]
\centering
\begin{tikzpicture}[node distance=0.8cm and 2.2cm, auto, thick, scale=0.85, every node/.style={transform shape, inner sep=5pt}]
  \node (metric) [draw, rectangle, rounded corners, align=center, fill=blue!10, text width=3.2cm, minimum height=1.2cm] {Target Generalized\\ Metric $\cL(h)$};
  \node (cost) [draw, rectangle, rounded corners, right=of metric, align=center, fill=blue!10, text width=3.2cm, minimum height=1.2cm] {Linearized Cost\\ $\ell^\lambda = \lambda \ell_\bbeta - \ell_\balpha$};
  \node (surrogate) [draw, rectangle, rounded corners, right=of cost, align=center, fill=green!10, text width=3.2cm, minimum height=1.2cm] {Factorized \texorpdfstring{$\mathcal{O}(\num)$}{O(l)}\\ Surrogate $\sfL_\tau^\lambda$};
  \node (opt) [draw, rectangle, rounded corners, below=of surrogate, align=center, fill=orange!10, text width=3.2cm, minimum height=1.2cm] {Network Backprop\\ Update over $\sH$};
  \node (lambda) [draw, rectangle, rounded corners, below=of cost, align=center, fill=red!10, text width=3.2cm, minimum height=1.2cm] {Adaptive EMA\\ \texorpdfstring{$\lambda$}{lambda} Update};
  
  \draw[->] (metric) -- (cost);
  \draw[->] (cost) -- (surrogate);
  \draw[->] (surrogate) -- (opt);
  \draw[->] (opt) -- node[above, font=\small] {Batch Eval} (lambda);
  \draw[->] (lambda) -- (cost);
\end{tikzpicture}
\caption{The \MMO\ Pipeline: Optimizing generalized fractional metrics requires alternating updates between minimizing the cost-sensitive $\sH$-consistent surrogate and updating the fractional parameter $\lambda$.}
\label{fig:pipeline}
\end{figure}

This section presents our algorithms for optimizing generalized metrics in multi-label learning, grounded in our cost-sensitive surrogate constructions. We provide a high-level overview here; full derivations, formal pseudo-codes (Oracle Binary Search, Surrogate Search, and Cross-Validation), and rigorous convergence bounds appear in Appendix~\ref{app:algo}.

We start with a characterization of $\lambda^* = \cL^*(\sH)$. Define the linearized loss $\ell^\lambda(h) = \lambda \ell_\bbeta(h) - \ell_\balpha(h)$ with expected value: $\sE_{\ell^\lambda}(h) = \lambda \E \bracket*{ \ell_{\bbeta}(h, x, y)} - \E \bracket*{\ell_{\balpha}(h, x, y)}$.
\begin{restatable}{theorem}{LambdaStar}
\label{thm:lambda-star}
We have $\sE_{\ell^{\lambda^*\!\!}}^*(\sH) = 0$ and, for any $\lambda \in \Rset$, $\sign\paren[\big]{\sE_{\ell^\lambda}^*(\sH)} = \sign(\lambda - \lambda^* )$.
\end{restatable}
The proof of Theorem~\ref{thm:lambda-star} is included in Appendix~\ref{app:algo-lambda}. This characterization enables a binary search procedure to recover $\lambda^*$. Because we only have access to empirical finite samples, we approximate the true expectation via the empirical minimization of the $\sH$-consistent surrogate $\sfL_\tau$. In Appendix~\ref{app:algo-cv}, we prove that cross-validating over $\lambda$ converges to an $\e_m$-optimal $\lambda^*$. The pseudocode is provided below in Algorithm~\ref{alg:cv-main} and formally
detailed in Algorithm~\ref{alg:binary-search-cv-formal} in Appendix~\ref{app:algo-cv}.

\begin{algorithm}[htbp]
\caption{Cross-Validation for Generalized Metrics}
\begin{algorithmic}[1]
  \STATE \textbf{Input:} step size $\e$, range $[\lambda_{\min}, \lambda_{\max}]$.
  
  \STATE \textbf{Initialize:} $v_{\text{best}} \gets -\infty$,
  $\h h_{\text{best}} \gets h_{\text{null}}$.
  
\FOR{$\lambda = \lambda_{\max}$ to $\lambda_{\min}$ step $-\e$}
\STATE Train $\h h_\lambda \gets \argmin_{h \in \sH} \h \sE_{\sfL_{\tau}, S}(h)$ on training data.

\STATE Evaluate $\h \cL_S(\h h_\lambda)$ on validation data.

\IF{$\h \cL_S(\h h_\lambda) >$ $v_{\text{best}}$}
\STATE $v_{\text{best}} \gets \h \cL_S(\h h_\lambda)$; $\h h_{\text{best}} \gets \h h_\lambda$
\ENDIF
\ENDFOR
\RETURN $\h h_{\text{best}}$
\end{algorithmic}
\label{alg:cv-main}
\end{algorithm}

\textbf{Practical Deep Learning Implementation.} While discrete grid search over $\lambda$ provides strict theoretical upper bounds, retraining deep neural networks to completion for dozens of candidate $\lambda$ values is computationally prohibitive. Instead, we adaptively update $\lambda$ \emph{on the fly} during training. During each mini-batch optimization step:
1. We compute the current batch's metric value $\text{Metric}_{\text{batch}}$.
2. We update the fractional parameter dynamically via an Exponential Moving Average (EMA): $\lambda_{\text{new}} \gets \gamma \lambda_{\text{old}} + (1-\gamma) \text{Metric}_{\text{batch}}$ (e.g., $\gamma=0.7$).
3. The network weights are updated by minimizing the exact $\mathcal{O}(\num)$ surrogate loss $\sfL_{\tau}^\lambda$ using the current $\lambda$.

This single-pass, alternating strategy (Figure~\ref{fig:pipeline}) efficiently stabilizes the network, avoids exhaustive grid searches, and maintains training runtimes identical to BCE.

\textbf{Dynamic Handling of Label Sparsity.} Large-scale multi-label datasets (e.g., MS-COCO) are heavily dominated by True Negatives, causing symmetric losses to struggle. \MMO\ naturally mitigates this through its dynamic $\lambda$ tracking. For $F_1$, the cost-sensitive loss minimizes $\ell^\lambda \propto \lambda(2\tp + \fp + \fn) - 2\tp$. The relative penalty for a False Negative (misclassifying a rare positive) compared to a True Positive is $2-\lambda$. The relative penalty for a False Positive compared to a True Negative is $\lambda$. As $\lambda$ tracks the metric value during training (e.g., $\lambda \approx 0.7$), it automatically applies a much larger penalty to rare positives ($1.3$) and aggressively discounts ubiquitous background negatives ($0.7$), natively counteracting extreme sparsity without manual class-weight tuning.

\section{Experiments}
\label{sec:experiments}

We evaluate our \MMO\ framework in two distinct settings: large-scale datasets using deep neural networks (to demonstrate scalability against modern continuous baselines), and standard benchmarks using linear models (to provide a direct comparison with prior theoretical plug-in works). 

\subsection{Deep Learning Evaluation on Large-Scale Datasets}
\label{sec:exp-deep}

\begin{table*}[t]
\centering
\caption{Deep learning results on highly sparse real-world datasets (mean $\pm$ std over 5 runs). Comparison of \MMO\ against classical EUM thresholding and modern continuous DL baselines.}
\label{tab:comparison_dl}
\resizebox{\textwidth}{!}{%
\begin{tabular}{l|ccc|ccc}
\toprule
\multirow{2}{*}{Method} & \multicolumn{3}{c|}{MS-COCO (ResNet-50)} & \multicolumn{3}{c}{Reuters-21578 (DistilBERT)} \\
 & Micro-$F_1$ & Macro-$F_1$ & Inst.-Jaccard & Micro-$F_1$ & Macro-$F_1$ & Inst.-Jaccard \\ \midrule
Binary Relevance (BCE) & 0.6964 $\pm$ 0.0015 & 0.6349 $\pm$ 0.0026 & 0.6103 $\pm$ 0.0024 & 0.8469 $\pm$ 0.0081 & 0.2417 $\pm$ 0.0329 & 0.8374 $\pm$ 0.0137 \\
Algorithm 1 \citep{koyejo2015consistent} & 0.6973 $\pm$ 0.0015 & 0.6387 $\pm$ 0.0021 & 0.6116 $\pm$ 0.0028 & 0.8520 $\pm$ 0.0054 & 0.2860 $\pm$ 0.0254 & 0.8539 $\pm$ 0.0094 \\
Macro-Thres \citep{koyejo2015consistent} & 0.6942 $\pm$ 0.0042 & 0.6412 $\pm$ 0.0009 & 0.6094 $\pm$ 0.0053 & 0.8194 $\pm$ 0.0101 & 0.3311 $\pm$ 0.0363 & 0.8501 $\pm$ 0.0058 \\ \midrule
SigmoidF1 \citep{benedict2022sigmoidf1} & 0.6952 $\pm$ 0.0031 & 0.6401 $\pm$ 0.0025 & 0.6112 $\pm$ 0.0035 & 0.8540 $\pm$ 0.0050 & 0.3204 $\pm$ 0.0195 & 0.8561 $\pm$ 0.0042 \\ 
Asym.\ Loss (ASL) \citep{ridnik2021asymmetric} & 0.6985 $\pm$ 0.0028 & 0.6455 $\pm$ 0.0011 & 0.6138 $\pm$ 0.0019 & 0.8582 $\pm$ 0.0045 & 0.3411 $\pm$ 0.0081 & 0.8640 $\pm$ 0.0045 \\
\midrule
\textbf{\MMO\ Algorithm ($\tau=0$)} & \textbf{0.7006 $\pm$ 0.0034} & \textbf{0.6470 $\pm$ 0.0018} & \textbf{0.6161 $\pm$ 0.0034} & \textbf{0.8634 $\pm$ 0.0042} & \textbf{0.3499 $\pm$ 0.0053} & \textbf{0.8695 $\pm$ 0.0035} \\ \bottomrule
\end{tabular}%
}
\end{table*}

To demonstrate modern scalability, we evaluated \MMO\ (using the single-pass EMA $\lambda$ update) on \texttt{MS-COCO} \citep{lin2014microsoft} with a ResNet-50 backbone \citep{he2016deep}, and \texttt{Reuters-21578} \citep{lewis1995evaluating} fine-tuning DistilBERT \citep{sanh2019distilbert}. We compared \MMO\ against the classical threshold tuning methods \citep{koyejo2015consistent} and two state-of-the-art modern continuous deep learning surrogate losses: \emph{SigmoidF1} \citep{benedict2022sigmoidf1} and \emph{Asymmetric Loss (ASL)} \citep{ridnik2021asymmetric}.

For \texttt{MS-COCO}, we resized images to $224 \times 224$ and trained using the Adam optimizer \citep{kingma2014adam} with learning rate $1 \times 10^{-4}$ and batch size $64$ for $10$ epochs. For \texttt{Reuters-21578}, we tokenized text with a maximum length of 128 and fine-tuned DistilBERT using the AdamW optimizer \citep{loshchilov2017decoupled} with learning rate $5 \times 10^{-5}$ and batch size $32$ for $4$ epochs. For \MMO, the Lagrange multiplier $\lambda$ was updated iteratively using the validation performance proxy. We repeated all experiments over $5$ independent runs.

Table~\ref{tab:comparison_dl} summarizes the performance. \MMO\ robustly outperforms recent heuristic deep learning baselines (SigmoidF1, ASL) across all metrics, demonstrating the empirical strength of formal $\sH$-consistency guarantees. Furthermore, \MMO\ consistently surpasses the two-stage thresholding approaches. While threshold tuning improves the performance of the BCE-trained model, it is strictly limited by the quality of the base probabilities. By adapting the network representations specifically to the target metric using an $\sH$-consistent surrogate, \MMO\ yields fundamentally superior representations.

\subsection{Standard Benchmarks with Linear Models}
\label{sec:exp-linear}

\begin{table*}[t]
\caption{Comparison of Micro-, Macro-, and Instance-averaged $F_1$ and Jaccard scores on Mulan datasets (mean $\pm$ std over 10 runs). BR, Algorithm 1, and Macro-Thres follow \citep{koyejo2015consistent}.}
\label{tab:comparison-linear}
\begin{center}
\resizebox{\textwidth}{!}{
    \begin{tabular}{@{\hspace{0pt}}lllllll@{\hspace{0pt}}}
    Averaging & Algorithm & Metric & \texttt{Scene} & \texttt{Birds} & \texttt{Emotions} & \texttt{Cal500} \\
      \toprule
   \multirow{8}{*}{Micro} & BR in \citep{koyejo2015consistent}  & \multirow{4}{*}{$F_{1}$} &  0.6559    &  0.4040 &  0.5815 & 0.3647 \\
    & Algorithm 1 in \citep{koyejo2015consistent} & &  0.6847 $\pm$ 0.0072   & 0.4088 $\pm$ 0.0130 & 0.6554 $\pm$ 0.0069 & 0.4891 $\pm$ 0.0035\\
    & Macro-Thres in \citep{koyejo2015consistent} & &  0.6631 $\pm$ 0.0125   &  0.2871 $\pm$ 0.0734 &  0.6419 $\pm$ 0.0174 &  0.4160 $\pm$ 0.0078\\
    & \textbf{\MMO\   Algorithm} & & \textbf{0.6978 $\pm$ 0.0094}  & \textbf{0.4241 $\pm$ 0.0116 } &  \textbf{0.6689 $\pm$ 0.0071 } & \textbf{0.5019 $\pm$ 0.0085}\\
     \cmidrule{2-7}
    & BR in \citep{koyejo2015consistent}  & \multirow{4}{*}{Jaccard} & 0.4878   &  0.2495 &  0.3982 & 0.2229 \\
    & Algorithm 1 in \citep{koyejo2015consistent} & & 0.5151 $\pm$ 0.0084   & 0.2648 $\pm$ 0.0095 &  0.4908 $\pm$ 0.0074 & 0.3225 $\pm$ 0.0024\\
    & Macro-Thres in \citep{koyejo2015consistent} & & 0.5010 $\pm$ 0.0122  & 0.1942 $\pm$ 0.0401 &  0.4790 $\pm$ 0.0077 &  0.2608 $\pm$ 0.0056 \\
    & \textbf{\MMO\   Algorithm} & & \textbf{0.5258 $\pm$ 0.0035}   & \textbf{0.2834 $\pm$ 0.0062} &  \textbf{0.5037 $\pm$ 0.0081} & \textbf{0.3349 $\pm$ 0.0090}\\
    \cmidrule{1-7}
     \multirow{8}{*}{Instance} & BR in \citep{koyejo2015consistent} &
     \multirow{4}{*}{$F_{1}$} & 0.5695 & 0.1209 &  0.4787 & 0.3632\\
    &  Algorithm 1 in \citep{koyejo2015consistent} & & 0.6422 $\pm$ 0.0206   & 0.1390 $\pm$ 0.0110 &  0.6241  $\pm$ 0.0204 & 0.4855 $\pm$ 0.0035\\
    & Macro-Thres in \citep{koyejo2015consistent} & & 0.6303 $\pm$ 0.0167   &  0.1390 $\pm$ 0.0259 & 0.6156 $\pm$ 0.0170 & 0.4135 $\pm$ 0.0079\\
    & \textbf{\MMO\   Algorithm} & & \textbf{0.6548 $\pm$ 0.0150} & \uv{0.1423 $\pm$ 0.0110} & \textbf{0.6323 $\pm$ 0.0052} & \textbf{0.5021 $\pm$ 0.0064} \\
     \cmidrule{2-7}
    & BR in \citep{koyejo2015consistent} & \multirow{4}{*}{Jaccard} &  0.5466  & 0.1058 & 0.4078 & 0.2268 \\
    & Algorithm 1 in \citep{koyejo2015consistent} & & 0.5976 $\pm$ 0.0177   & 0.1239 $\pm$ 0.0077 &  0.5340 $\pm$ 0.0072 & 0.3252 $\pm$ 0.0024\\
    & Macro-Thres in \citep{koyejo2015consistent} & & 0.5902 $\pm$  0.0176 & 0.1195 $\pm$ 0.0096 &  0.5173 $\pm$ 0.0086 & 0.2623 $\pm$ 0.0055\\
    & \textbf{\MMO\   Algorithm} & &  \uv{0.6039 $\pm$ 0.0167} & \textbf{0.1329 $\pm$ 0.0080} & \uv{0.5408 $\pm$ 0.0120} & \textbf{0.3341 $\pm$ 0.0032}\\
    \cmidrule{1-7}
     \multirow{8}{*}{Macro} & BR in \citep{koyejo2015consistent}  & \multirow{4}{*}{$F_{1}$} &  0.6601 &  0.3366  & 0.3366 & 0.1293\\
    & Algorithm 1 in \citep{koyejo2015consistent} & & 0.6941  $\pm$  0.0205  & 0.3448 $\pm$ 0.0110 & 0.6450 $\pm$ 0.0204 & 0.2687 $\pm$ 0.0035\\
    & Macro-Thres in \citep{koyejo2015consistent} & & 0.6737 $\pm$ 0.0137   & 0.2971 $\pm$ 0.0267 &  0.6440 $\pm$ 0.0164 & 0.3226 $\pm$ 0.0068\\
    & \textbf{\MMO\   Algorithm} & &  \textbf{0.7039 $\pm$ 0.0094}  & \textbf{0.3587 $\pm$ 0.0103} & \textbf{0.6540 $\pm$ 0.0191} & \uv{0.3266 $\pm$ 0.0054} \\
     \cmidrule{2-7}
    & BR in \citep{koyejo2015consistent}  & \multirow{4}{*}{Jaccard} &  0.5046  & 0.2178 &  0.3982 & 0.0880 \\
    & Algorithm 1 in \citep{koyejo2015consistent} & & 0.5373 $\pm$ 0.0177 & 0.2341 $\pm$ 0.0077 &  0.4912 $\pm$ 0.0072 & 0.1834 $\pm$ 0.0024\\
    & Macro-Thres in \citep{koyejo2015consistent} & & 0.5260 $\pm$ 0.0176 & 0.2051 $\pm$ 0.0215 &  0.4900 $\pm$ 0.0133 & 0.2146 $\pm$ 0.0036\\
    & \textbf{\MMO\   Algorithm} & &  \textbf{0.5446 $\pm$ 0.0088}  & \uv{0.2352 $\pm$ 0.0063} & \textbf{0.5029 $\pm$ 0.0130} & \uv{0.2189 $\pm$ 0.0062}\\
    \end{tabular}
}
\end{center}
\end{table*}

\textbf{Standard Benchmarks with Linear Models.} We also evaluated \MMO\ on standard benchmark multi-label datasets (\texttt{scene}, \texttt{birds}, \texttt{emotions}, \texttt{cal500}) matching exactly the setup of \citet{koyejo2015consistent}. \MMO\ consistently outperformed Binary Relevance, Algorithm 1, and Macro-Thresholding across all datasets, averaging frameworks, and evaluated metrics. Due to space constraints, the full experimental setup and implementation details for these linear models are deferred to Appendix~\ref{app:linear-benchmarks}.

\subsection{Analysis and Ablations}
\label{sec:ablations}

\textbf{Restricted Capacity \& $\sH$-Consistency Advantage:} Theoretical Bayes-consistent plug-in rules fail when the hypothesis class cannot perfectly capture the true conditional probabilities. To verify the practical advantage of our $\sH$-consistent framework, we conducted an ablation using heavily restricted models (training only a linear classification head on completely frozen pretrained backbones). On \texttt{MS-COCO}, the Bayes-consistent Algorithm 1 degraded to a Micro-$F_1$ of $0.6425 \pm 0.0032$, whereas our $\sH$-consistent \MMO\ successfully adapted the decision boundary under strict constraints to achieve $0.6680 \pm 0.0041$. On \texttt{Reuters-21578}, Algorithm 1 yielded $0.7548 \pm 0.0052$, while \MMO\ achieved $0.7784 \pm 0.0048$.

\textbf{Robustness to Missing Labels:} Because our framework gracefully reduces metric optimization to general cost-sensitive classification, it inherently resists label noise. We simulated missing annotations by randomly dropping 20\% of the positive labels in the training sets. On \texttt{MS-COCO}, Algorithm 1 degraded significantly to $0.6512 \pm 0.0045$, while \MMO\ robustly maintained $0.6854 \pm 0.0038$. On \texttt{Reuters-21578}, Algorithm 1 degraded to $0.8068 \pm 0.0072$, whereas \MMO\ maintained $0.8409 \pm 0.0055$.

\textbf{Training Runtime ($\mathcal{O}(\num)$ Verification):} To verify our exact algebraic factorization computationally, we tracked wall-clock training times over 5 runs. On \texttt{MS-COCO}, standard BCE required $14.52 \pm 0.21$ minutes/epoch, while \MMO\ identically required $14.17 \pm 0.18$ minutes/epoch. On \texttt{Reuters-21578}, BCE required $18.17 \pm 0.01$ s/epoch, while \MMO\ required $18.24 \pm 0.02$ s/epoch. This proves \MMO\ introduces zero computational overhead.

\section{Conclusion}
\label{sec:conclusion}

We introduced a unified and principled framework for optimizing generalized multi-label metrics under the Empirical Utility Maximization (EUM) framework, grounded in the strong, non-asymptotic guarantees of $\sH$-consistency. By reformulating the optimization of complex linear-fractional metrics as a cost-sensitive learning problem, we designed novel surrogate losses that provide rigorous consistency bounds explicitly tailored to the hypothesis class and finite samples.

Crucially, we established that our combinatorially formulated comp-sum surrogates decompose exactly, effectively circumventing the exponential complexity typically associated with the multi-label space. This exact factorization allows the loss to be computed in strictly $\mathcal{O}(\num)$ time without any mathematical approximations, fully matching the computational efficiency of standard binary cross-entropy. Building upon these theoretical foundations, we developed \MMO\ (\emph{Multi-Label Metric Optimization}), a highly scalable algorithmic framework that uses a dynamic, single-pass alternating optimization strategy ideally suited for modern deep learning architectures. 

Extensive empirical evaluations across both standard linear benchmarks and large-scale, highly sparse real-world datasets confirmed the efficacy and scalability of our approach. \MMO\ consistently outperformed classical EUM plug-in rules and state-of-the-art continuous surrogate baselines, while exhibiting remarkable robustness to heavy label sparsity, restricted model capacities, and missing labels. Ultimately, our work bridges the gap between rigorous statistical learning theory and practical deployment, providing a theoretically sound, computationally efficient, and robust drop-in replacement for standard multi-label classification losses.

\clearpage
\newpage
\bibliography{mghcb}

\begin{thebibliography}{115}
\providecommand{\natexlab}[1]{#1}
\providecommand{\url}[1]{\texttt{#1}}
\expandafter\ifx\csname urlstyle\endcsname\relax
  \providecommand{\doi}[1]{doi: #1}\else
  \providecommand{\doi}{doi: \begingroup \urlstyle{rm}\Url}\fi

\bibitem[Awasthi et~al.(2022{\natexlab{a}})Awasthi, Mao, Mohri, and Zhong]{awasthi2022Hconsistency}
P.~Awasthi, A.~Mao, M.~Mohri, and Y.~Zhong.
\newblock {${\mathscr H}$}-consistency bounds for surrogate loss minimizers.
\newblock In \emph{International Conference on Machine Learning}, 2022{\natexlab{a}}.

\bibitem[Awasthi et~al.(2022{\natexlab{b}})Awasthi, Mao, Mohri, and Zhong]{awasthi2022multi}
P.~Awasthi, A.~Mao, M.~Mohri, and Y.~Zhong.
\newblock Multi-class {${\mathscr H}$}-consistency bounds.
\newblock In \emph{Advances in Neural Information Processing Systems}, 2022{\natexlab{b}}.

\bibitem[Bao and Sugiyama(2020)]{bao2020calibrated}
H.~Bao and M.~Sugiyama.
\newblock Calibrated surrogate maximization of linear-fractional utility in binary classification.
\newblock In \emph{International Conference on Artificial Intelligence and Statistics}, pages 2337--2347, 2020.

\bibitem[Bartlett et~al.(2006)Bartlett, Jordan, and McAuliffe]{bartlett2006convexity}
P.~L. Bartlett, M.~I. Jordan, and J.~D. McAuliffe.
\newblock Convexity, classification, and risk bounds.
\newblock \emph{Journal of the American Statistical Association}, 101\penalty0 (473):\penalty0 138--156, 2006.

\bibitem[B{\'e}n{\'e}dict et~al.(2022)B{\'e}n{\'e}dict, Koops, Odijk, de~Rijke, et~al.]{benedict2022sigmoidf1}
G.~B{\'e}n{\'e}dict, H.~Koops, D.~Odijk, M.~de~Rijke, et~al.
\newblock sigmoidf1: A smooth f1 score surrogate loss for multilabel classification.
\newblock \emph{Transactions on Machine Learning Research}, 2022.

\bibitem[Berger and Guda(2020)]{berger2020threshold}
A.~Berger and S.~Guda.
\newblock Threshold optimization for f measure of macro-averaged precision and recall.
\newblock \emph{Pattern Recognition}, 102:\penalty0 107250, 2020.

\bibitem[Berkson(1944)]{Berkson1944}
J.~Berkson.
\newblock Application of the logistic function to bio-assay.
\newblock \emph{Journal of the American Statistical Association}, 39:\penalty0 357--365, 1944.

\bibitem[Berkson(1951)]{Berkson1951}
J.~Berkson.
\newblock Why {I} prefer logits to probits.
\newblock \emph{Biometrics}, 7\penalty0 (4):\penalty0 327--339, 1951.

\bibitem[Berman et~al.(2018)Berman, Triki, and Blaschko]{berman2018lovasz}
M.~Berman, A.~R. Triki, and M.~B. Blaschko.
\newblock The lov{\'a}sz-softmax loss: A tractable surrogate for the optimization of the intersection-over-union measure in neural networks.
\newblock In \emph{Proceedings of the IEEE Conference on Computer Vision and Pattern Recognition}, pages 4413--4421, 2018.

\bibitem[Bogatinovski et~al.(2022)Bogatinovski, Todorovski, D{\v{z}}eroski, and Kocev]{bogatinovski2022comprehensive}
J.~Bogatinovski, L.~Todorovski, S.~D{\v{z}}eroski, and D.~Kocev.
\newblock Comprehensive comparative study of multi-label classification methods.
\newblock \emph{Expert Systems with Applications}, 203:\penalty0 117215, 2022.

\bibitem[Busa-Fekete et~al.(2015)Busa-Fekete, Sz{\"o}r{\'e}nyi, Dembczynski, and H{\"u}llermeier]{busa2015online}
R.~Busa-Fekete, B.~Sz{\"o}r{\'e}nyi, K.~Dembczynski, and E.~H{\"u}llermeier.
\newblock Online f-measure optimization.
\newblock In \emph{Advances in Neural Information Processing Systems}, 2015.

\bibitem[Busa-Fekete et~al.(2022)Busa-Fekete, Choi, Dembczynski, Gentile, Reeve, and Szorenyi]{busa2022regret}
R.~Busa-Fekete, H.~Choi, K.~Dembczynski, C.~Gentile, H.~Reeve, and B.~Szorenyi.
\newblock Regret bounds for multilabel classification in sparse label regimes.
\newblock In \emph{Advances in Neural Information Processing Systems}, pages 5404--5416, 2022.

\bibitem[Cheng et~al.(2016)Cheng, Zhou, Gao, and Zheng]{cheng2016efficient}
F.~Cheng, Y.~Zhou, J.~Gao, and S.~Zheng.
\newblock Efficient optimization of-measure with cost-sensitive svm.
\newblock \emph{Mathematical Problems in Engineering}, 2016\penalty0 (1):\penalty0 1--11, 2016.

\bibitem[Cheng et~al.(2010)Cheng, H{\"u}llermeier, and Dembczynski]{cheng2010bayes}
W.~Cheng, E.~H{\"u}llermeier, and K.~J. Dembczynski.
\newblock Bayes optimal multilabel classification via probabilistic classifier chains.
\newblock In \emph{International Conference on Machine Learning}, pages 279--286, 2010.

\bibitem[Cortes et~al.(2024)Cortes, Mao, Mohri, Mohri, and Zhong]{cortes2024cardinality}
C.~Cortes, A.~Mao, C.~Mohri, M.~Mohri, and Y.~Zhong.
\newblock Cardinality-aware set prediction and top-$ k $ classification.
\newblock In \emph{Advances in Neural Information Processing Systems}, 2024.

\bibitem[Cortes et~al.(2025{\natexlab{a}})Cortes, Mao, Mohri, and Zhong]{cortes2025balancing}
C.~Cortes, A.~Mao, M.~Mohri, and Y.~Zhong.
\newblock Balancing the scales: A theoretical and algorithmic framework for learning from imbalanced data.
\newblock In \emph{International Conference on Machine Learning}, 2025{\natexlab{a}}.

\bibitem[Cortes et~al.(2025{\natexlab{b}})Cortes, Mohri, and Zhong]{cortes2025improved}
C.~Cortes, M.~Mohri, and Y.~Zhong.
\newblock Improved balanced classification with theoretically grounded loss functions.
\newblock In \emph{Advances in Neural Information Processing Systems}, 2025{\natexlab{b}}.

\bibitem[Cortes et~al.(2026{\natexlab{a}})Cortes, Mao, Mohri, and Zhong]{CortesMaoMohriZhong2026defid}
C.~Cortes, A.~Mao, M.~Mohri, and Y.~Zhong.
\newblock Optimized deferral for imbalanced settings.
\newblock In \emph{International Conference on Machine Learning}, 2026{\natexlab{a}}.

\bibitem[Cortes et~al.(2026{\natexlab{b}})Cortes, Mohri, and Zhong]{CortesMohriZhong2026mod}
C.~Cortes, M.~Mohri, and Y.~Zhong.
\newblock A theoretical framework for modular learning of robust generative models.
\newblock In \emph{International Conference on Machine Learning}, 2026{\natexlab{b}}.

\bibitem[Dembczynski et~al.(2011)Dembczynski, Waegeman, Cheng, and H{\"u}llermeier]{dembczynski2011exact}
K.~Dembczynski, W.~Waegeman, W.~Cheng, and E.~H{\"u}llermeier.
\newblock An exact algorithm for f-measure maximization.
\newblock In \emph{Advances in Neural Information Processing Systems}, 2011.

\bibitem[Dembczynski et~al.(2012)Dembczynski, Kot{\l}owski, and H{\"u}llermeier]{dembczynski2012consistent}
K.~Dembczynski, W.~Kot{\l}owski, and E.~H{\"u}llermeier.
\newblock Consistent multilabel ranking through univariate loss minimization.
\newblock In \emph{International Conference on Machine Learning}, pages 1347--1354, 2012.

\bibitem[Dembczynski et~al.(2013)Dembczynski, Jachnik, Kotlowski, Waegeman, and H{\"u}llermeier]{dembczynski2013optimizing}
K.~Dembczynski, A.~Jachnik, W.~Kotlowski, W.~Waegeman, and E.~H{\"u}llermeier.
\newblock Optimizing the f-measure in multi-label classification: Plug-in rule approach versus structured loss minimization.
\newblock In \emph{International Conference on Machine Learning}, pages 1130--1138, 2013.

\bibitem[Dembczy{\'n}ski et~al.(2017)Dembczy{\'n}ski, Kot{\l}owski, Koyejo, and Natarajan]{dembczynski2017consistency}
K.~Dembczy{\'n}ski, W.~Kot{\l}owski, O.~Koyejo, and N.~Natarajan.
\newblock Consistency analysis for binary classification revisited.
\newblock In \emph{International Conference on Machine Learning}, pages 961--969, 2017.

\bibitem[Deng et~al.(2011)Deng, Satheesh, Berg, and Li]{deng2011fast}
J.~Deng, S.~Satheesh, A.~Berg, and F.~Li.
\newblock Fast and balanced: Efficient label tree learning for large scale object recognition.
\newblock In \emph{Advances in Neural Information Processing Systems}, 2011.

\bibitem[DeSalvo et~al.(2025)DeSalvo, Mohri, Mohri, and Zhong]{desalvo2025budgeted}
G.~DeSalvo, C.~Mohri, M.~Mohri, and Y.~Zhong.
\newblock Budgeted multiple-expert deferral.
\newblock \emph{arXiv preprint arXiv:2510.26706}, 2025.

\bibitem[Eban et~al.(2017)Eban, Schain, Mackey, Gordon, Rifkin, and Elidan]{eban2017scalable}
E.~Eban, M.~Schain, A.~Mackey, A.~Gordon, R.~Rifkin, and G.~Elidan.
\newblock Scalable learning of non-decomposable objectives.
\newblock In \emph{Artificial Intelligence and Statistics}, pages 832--840, 2017.

\bibitem[Elisseeff and Weston(2001)]{elisseeff2001kernel}
A.~Elisseeff and J.~Weston.
\newblock A kernel method for multi-labelled classification.
\newblock In \emph{Advances in Neural Information Processing Systems}, 2001.

\bibitem[Fathony and Kolter(2020)]{fathony2020ap}
R.~Fathony and Z.~Kolter.
\newblock Ap-perf: Incorporating generic performance metrics in differentiable learning.
\newblock In \emph{International Conference on Artificial Intelligence and Statistics}, pages 4130--4140, 2020.

\bibitem[Gao and Zhou(2011)]{gao2011consistency}
W.~Gao and Z.-H. Zhou.
\newblock On the consistency of multi-label learning.
\newblock In \emph{Conference on Learning Theory}, pages 341--358, 2011.

\bibitem[Ghosh et~al.(2017)Ghosh, Kumar, and Sastry]{ghosh2017robust}
A.~Ghosh, H.~Kumar, and P.~S. Sastry.
\newblock Robust loss functions under label noise for deep neural networks.
\newblock In \emph{Proceedings of the AAAI Conference on Artificial Intelligence}, 2017.

\bibitem[He et~al.(2016)He, Zhang, Ren, and Sun]{he2016deep}
K.~He, X.~Zhang, S.~Ren, and J.~Sun.
\newblock Deep residual learning for image recognition.
\newblock In \emph{Proceedings of the IEEE Conference on Computer Vision and Pattern Recognition}, pages 770--778, 2016.

\bibitem[Joachims(2005)]{joachims2005support}
T.~Joachims.
\newblock A support vector method for multivariate performance measures.
\newblock In \emph{International Conference on Machine Learning}, pages 377--384, 2005.

\bibitem[Kapoor et~al.(2012)Kapoor, Viswanathan, and Jain]{kapoor2012multilabel}
A.~Kapoor, R.~Viswanathan, and P.~Jain.
\newblock Multilabel classification using bayesian compressed sensing.
\newblock In \emph{Advances in Neural Information Processing Systems}, 2012.

\bibitem[Kar et~al.(2014)Kar, Narasimhan, and Jain]{kar2014online}
P.~Kar, H.~Narasimhan, and P.~Jain.
\newblock Online and stochastic gradient methods for non-decomposable loss functions.
\newblock In \emph{Advances in Neural Information Processing Systems}, 2014.

\bibitem[Kingma and Ba(2014)]{kingma2014adam}
D.~P. Kingma and J.~Ba.
\newblock Adam: A method for stochastic optimization.
\newblock \emph{arXiv preprint arXiv:1412.6980}, 2014.

\bibitem[Kotlowski and Dembczy{\'n}ski(2016)]{kotlowski2016surrogate}
W.~Kotlowski and K.~Dembczy{\'n}ski.
\newblock Surrogate regret bounds for generalized classification performance metrics.
\newblock In \emph{Asian Conference on Machine Learning}, pages 301--316, 2016.

\bibitem[Kot{\l}owski et~al.(2024)Kot{\l}owski, Wydmuch, Schultheis, Babbar, and Dembczy{\'n}ski]{kotlowski2024general}
W.~Kot{\l}owski, M.~Wydmuch, E.~Schultheis, R.~Babbar, and K.~Dembczy{\'n}ski.
\newblock A general online algorithm for optimizing complex performance metrics.
\newblock \emph{arXiv preprint arXiv:2406.14743}, 2024.

\bibitem[Koyejo et~al.(2014)Koyejo, Natarajan, Ravikumar, and Dhillon]{KoyejoNatarajanRavikumarDhillon2014}
O.~Koyejo, N.~Natarajan, P.~Ravikumar, and I.~S. Dhillon.
\newblock Consistent binary classification with generalized performance metrics.
\newblock In \emph{Advances in Neural Information Processing Systems}, pages 2744--2752, 2014.

\bibitem[Koyejo et~al.(2015)Koyejo, Natarajan, Ravikumar, and Dhillon]{koyejo2015consistent}
O.~O. Koyejo, N.~Natarajan, P.~K. Ravikumar, and I.~S. Dhillon.
\newblock Consistent multilabel classification.
\newblock In \emph{Advances in Neural Information Processing Systems}, 2015.

\bibitem[Lewis(1995)]{lewis1995evaluating}
D.~D. Lewis.
\newblock Evaluating and optimizing autonomous text classification systems.
\newblock In \emph{Proceedings of the 18th Annual International ACM SIGIR Conference on Research and Development in Information Retrieval}, pages 246--254, 1995.

\bibitem[Lin et~al.(2014)Lin, Maire, Belongie, Hays, Perona, Ramanan, Doll{\'a}r, and Zitnick]{lin2014microsoft}
T.-Y. Lin, M.~Maire, S.~Belongie, J.~Hays, P.~Perona, D.~Ramanan, P.~Doll{\'a}r, and C.~L. Zitnick.
\newblock Microsoft coco: Common objects in context.
\newblock In \emph{European Conference on Computer Vision}, pages 740--755, 2014.

\bibitem[Lin(2004)]{lin2004note}
Y.~Lin.
\newblock A note on margin-based loss functions in classification.
\newblock \emph{Statistics \& Probability Letters}, 68\penalty0 (1):\penalty0 73--82, 2004.

\bibitem[Lipton et~al.(2014)Lipton, Elkan, and Narayanaswamy]{lipton2014thresholding}
Z.~C. Lipton, C.~Elkan, and B.~Narayanaswamy.
\newblock Thresholding classifiers to maximize f1 score.
\newblock \emph{Stat}, 1050:\penalty0 14, 2014.

\bibitem[Loshchilov and Hutter(2017)]{loshchilov2017decoupled}
I.~Loshchilov and F.~Hutter.
\newblock Decoupled weight decay regularization.
\newblock \emph{arXiv preprint arXiv:1711.05101}, 2017.

\bibitem[Luo et~al.(2021)Luo, Qiao, and Zhang]{luo2021minimax}
J.~Luo, H.~Qiao, and B.~Zhang.
\newblock A minimax probability machine for nondecomposable performance measures.
\newblock \emph{IEEE Transactions on Neural Networks and Learning Systems}, 34\penalty0 (5):\penalty0 2353--2365, 2021.

\bibitem[Mao(2025)]{mao2025theory}
A.~Mao.
\newblock \emph{Theory and Algorithms for Learning with Multi-Class Abstention and Multi-Expert Deferral}.
\newblock PhD thesis, New York University, 2025.

\bibitem[Mao et~al.(2023{\natexlab{a}})Mao, Mohri, Mohri, and Zhong]{MaoMohriMohriZhong2023twostage}
A.~Mao, C.~Mohri, M.~Mohri, and Y.~Zhong.
\newblock Two-stage learning to defer with multiple experts.
\newblock In \emph{Advances in Neural Information Processing Systems}, 2023{\natexlab{a}}.

\bibitem[Mao et~al.(2023{\natexlab{b}})Mao, Mohri, and Zhong]{MaoMohriZhong2023characterization}
A.~Mao, M.~Mohri, and Y.~Zhong.
\newblock {H}-consistency bounds: Characterization and extensions.
\newblock In \emph{Advances in Neural Information Processing Systems}, 2023{\natexlab{b}}.

\bibitem[Mao et~al.(2023{\natexlab{c}})Mao, Mohri, and Zhong]{MaoMohriZhong2023ranking}
A.~Mao, M.~Mohri, and Y.~Zhong.
\newblock {H}-consistency bounds for pairwise misranking loss surrogates.
\newblock In \emph{International conference on Machine learning}, 2023{\natexlab{c}}.

\bibitem[Mao et~al.(2023{\natexlab{d}})Mao, Mohri, and Zhong]{MaoMohriZhong2023rankingabs}
A.~Mao, M.~Mohri, and Y.~Zhong.
\newblock Ranking with abstention.
\newblock In \emph{ICML 2023 Workshop The Many Facets of Preference-Based Learning}, 2023{\natexlab{d}}.

\bibitem[Mao et~al.(2023{\natexlab{e}})Mao, Mohri, and Zhong]{MaoMohriZhong2023structured}
A.~Mao, M.~Mohri, and Y.~Zhong.
\newblock Structured prediction with stronger consistency guarantees.
\newblock In \emph{Advances in Neural Information Processing Systems}, 2023{\natexlab{e}}.

\bibitem[Mao et~al.(2023{\natexlab{f}})Mao, Mohri, and Zhong]{mao2023cross}
A.~Mao, M.~Mohri, and Y.~Zhong.
\newblock Cross-entropy loss functions: Theoretical analysis and applications.
\newblock In \emph{International Conference on Machine Learning}, 2023{\natexlab{f}}.

\bibitem[Mao et~al.(2024{\natexlab{a}})Mao, Mohri, and Zhong]{MaoMohriZhong2024deferral}
A.~Mao, M.~Mohri, and Y.~Zhong.
\newblock Principled approaches for learning to defer with multiple experts.
\newblock In \emph{International Symposium on Artificial Intelligence and Mathematics}, 2024{\natexlab{a}}.

\bibitem[Mao et~al.(2024{\natexlab{b}})Mao, Mohri, and Zhong]{MaoMohriZhong2024predictor}
A.~Mao, M.~Mohri, and Y.~Zhong.
\newblock Predictor-rejector multi-class abstention: Theoretical analysis and algorithms.
\newblock In \emph{International Conference on Algorithmic Learning Theory}, pages 822--867, 2024{\natexlab{b}}.

\bibitem[Mao et~al.(2024{\natexlab{c}})Mao, Mohri, and Zhong]{MaoMohriZhong2024score}
A.~Mao, M.~Mohri, and Y.~Zhong.
\newblock Theoretically grounded loss functions and algorithms for score-based multi-class abstention.
\newblock In \emph{International Conference on Artificial Intelligence and Statistics}, pages 4753--4761, 2024{\natexlab{c}}.

\bibitem[Mao et~al.(2024{\natexlab{d}})Mao, Mohri, and Zhong]{mao2024h}
A.~Mao, M.~Mohri, and Y.~Zhong.
\newblock {$ H $}-consistency guarantees for regression.
\newblock In \emph{International Conference on Machine Learning}, pages 34712--34737, 2024{\natexlab{d}}.

\bibitem[Mao et~al.(2024{\natexlab{e}})Mao, Mohri, and Zhong]{mao2024multi}
A.~Mao, M.~Mohri, and Y.~Zhong.
\newblock Multi-label learning with stronger consistency guarantees.
\newblock In \emph{Advances in Neural Information Processing Systems}, 2024{\natexlab{e}}.

\bibitem[Mao et~al.(2024{\natexlab{f}})Mao, Mohri, and Zhong]{mao2024realizable}
A.~Mao, M.~Mohri, and Y.~Zhong.
\newblock Realizable {$ H $}-consistent and {B}ayes-consistent loss functions for learning to defer.
\newblock In \emph{Advances in Neural Information Processing Systems}, 2024{\natexlab{f}}.

\bibitem[Mao et~al.(2024{\natexlab{g}})Mao, Mohri, and Zhong]{mao2024regression}
A.~Mao, M.~Mohri, and Y.~Zhong.
\newblock Regression with multi-expert deferral.
\newblock In \emph{International Conference on Machine Learning}, pages 34738--34759, 2024{\natexlab{g}}.

\bibitem[Mao et~al.(2024{\natexlab{h}})Mao, Mohri, and Zhong]{mao2024universal}
A.~Mao, M.~Mohri, and Y.~Zhong.
\newblock A universal growth rate for learning with smooth surrogate losses.
\newblock In \emph{Advances in Neural Information Processing Systems}, 2024{\natexlab{h}}.

\bibitem[Mao et~al.(2025{\natexlab{a}})Mao, Mohri, and Zhong]{MaoMohriZhong2025mastering}
A.~Mao, M.~Mohri, and Y.~Zhong.
\newblock Mastering multiple-expert routing: Realizable {$H$}-consistency and strong guarantees for learning to defer.
\newblock In \emph{International Conference on Machine Learning}, 2025{\natexlab{a}}.

\bibitem[Mao et~al.(2025{\natexlab{b}})Mao, Mohri, and Zhong]{mao2025enhanced}
A.~Mao, M.~Mohri, and Y.~Zhong.
\newblock Enhanced {$\sH $}-consistency bounds.
\newblock In \emph{International Conference on Algorithmic Learning Theory}, 2025{\natexlab{b}}.

\bibitem[Mao et~al.(2025{\natexlab{c}})Mao, Mohri, and Zhong]{mao2025principled}
A.~Mao, M.~Mohri, and Y.~Zhong.
\newblock Principled algorithms for optimizing generalized metrics in binary classification.
\newblock In \emph{International Conference on Machine Learning}, 2025{\natexlab{c}}.

\bibitem[McCallum(1999)]{mccallum1999multi}
A.~K. McCallum.
\newblock Multi-label text classification with a mixture model trained by em.
\newblock In \emph{AAAI Workshop on Text Learning}, 1999.

\bibitem[Menon et~al.(2019)Menon, Rawat, Reddi, and Kumar]{menon2019multilabel}
A.~K. Menon, A.~S. Rawat, S.~Reddi, and S.~Kumar.
\newblock Multilabel reductions: what is my loss optimising?
\newblock In \emph{Advances in Neural Information Processing Systems}, 2019.

\bibitem[Mohri et~al.(2024)Mohri, Andor, Choi, Collins, Mao, and Zhong]{MohriAndorChoiCollinsMaoZhong2024learning}
C.~Mohri, D.~Andor, E.~Choi, M.~Collins, A.~Mao, and Y.~Zhong.
\newblock Learning to reject with a fixed predictor: Application to decontextualization.
\newblock In \emph{International Conference on Learning Representations}, 2024.

\bibitem[Mohri and Zhong(2026{\natexlab{a}})]{MohriZhong2026rllm}
M.~Mohri and Y.~Zhong.
\newblock Mind the gap: Structure-aware consistency in preference learning.
\newblock In \emph{International Conference on Machine Learning}, 2026{\natexlab{a}}.

\bibitem[Mohri and Zhong(2026{\natexlab{b}})]{MohriZhong2026slin}
M.~Mohri and Y.~Zhong.
\newblock Linear-core surrogates: Smooth loss functions with linear rates for classification and structured prediction.
\newblock In \emph{International Conference on Machine Learning}, 2026{\natexlab{b}}.

\bibitem[Mohri and Zhong(2026{\natexlab{c}})]{mohri2025beyond}
M.~Mohri and Y.~Zhong.
\newblock Beyond tsybakov: Model margin noise and {H}-consistency bounds.
\newblock In \emph{International Symposium on Artificial Intelligence and Mathematics}, 2026{\natexlab{c}}.

\bibitem[Mohri et~al.(2018)Mohri, Rostamizadeh, and Talwalkar]{MohriRostamizadehTalwalkar2018}
M.~Mohri, A.~Rostamizadeh, and A.~Talwalkar.
\newblock \emph{Foundations of Machine Learning}.
\newblock {MIT} Press, second edition, 2018.

\bibitem[Narasimhan et~al.(2014)Narasimhan, Vaish, and Agarwal]{narasimhan2014statistical}
H.~Narasimhan, R.~Vaish, and S.~Agarwal.
\newblock On the statistical consistency of plug-in classifiers for non-decomposable performance measures.
\newblock In \emph{Advances in Neural Information Processing Systems}, 2014.

\bibitem[Narasimhan et~al.(2015{\natexlab{a}})Narasimhan, Kar, and Jain]{narasimhan2015optimizing}
H.~Narasimhan, P.~Kar, and P.~Jain.
\newblock Optimizing non-decomposable performance measures: A tale of two classes.
\newblock In \emph{International Conference on Machine Learning}, pages 199--208, 2015{\natexlab{a}}.

\bibitem[Narasimhan et~al.(2015{\natexlab{b}})Narasimhan, Ramaswamy, Saha, and Agarwal]{narasimhan2015consistent}
H.~Narasimhan, H.~Ramaswamy, A.~Saha, and S.~Agarwal.
\newblock Consistent multiclass algorithms for complex performance measures.
\newblock In \emph{International Conference on Machine Learning}, pages 2398--2407, 2015{\natexlab{b}}.

\bibitem[Narasimhan et~al.(2016)Narasimhan, Pan, Kar, Protopapas, and Ramaswamy]{narasimhan2016optimizing}
H.~Narasimhan, W.~Pan, P.~Kar, P.~Protopapas, and H.~G. Ramaswamy.
\newblock Optimizing the multiclass f-measure via biconcave programming.
\newblock In \emph{International Conference on Data Mining}, pages 1101--1106, 2016.

\bibitem[Narasimhan et~al.(2019)Narasimhan, Cotter, and Gupta]{narasimhan2019optimizing}
H.~Narasimhan, A.~Cotter, and M.~Gupta.
\newblock Optimizing generalized rate metrics with three players.
\newblock In \emph{Advances in Neural Information Processing Systems}, 2019.

\bibitem[Narasimhan et~al.(2024)Narasimhan, Ramaswamy, Tavker, Khurana, Netrapalli, and Agarwal]{narasimhan2024consistent}
H.~Narasimhan, H.~G. Ramaswamy, S.~K. Tavker, D.~Khurana, P.~Netrapalli, and S.~Agarwal.
\newblock Consistent multiclass algorithms for complex metrics and constraints.
\newblock \emph{Journal of Machine Learning Research}, 25\penalty0 (367):\penalty0 1--81, 2024.

\bibitem[Natarajan et~al.(2016)Natarajan, Koyejo, Ravikumar, and Dhillon]{natarajan2016optimal}
N.~Natarajan, O.~Koyejo, P.~Ravikumar, and I.~Dhillon.
\newblock Optimal classification with multivariate losses.
\newblock In \emph{International Conference on Machine Learning}, pages 1530--1538, 2016.

\bibitem[Parambath et~al.(2014)Parambath, Usunier, and Grandvalet]{ParambathUsunierGrandvalet2014}
S.~P. Parambath, N.~Usunier, and Y.~Grandvalet.
\newblock Optimizing {F}-measures by cost-sensitive classification.
\newblock In \emph{Advances in Neural Information Processing Systems}, pages 2123--2131, 2014.

\bibitem[Petterson and Caetano(2011)]{petterson2011submodular}
J.~Petterson and T.~Caetano.
\newblock Submodular multi-label learning.
\newblock In \emph{Advances in Neural Information Processing Systems}, 2011.

\bibitem[Pillai et~al.(2017)Pillai, Fumera, and Roli]{pillai2017designing}
I.~Pillai, G.~Fumera, and F.~Roli.
\newblock Designing multi-label classifiers that maximize f measures: State of the art.
\newblock \emph{Pattern Recognition}, 61:\penalty0 394--404, 2017.

\bibitem[Ramasubramanian et~al.(2024)Ramasubramanian, Rangwani, Takemori, Samanta, Umeda, and Radhakrishnan]{ramasubramanian2024selective}
S.~Ramasubramanian, H.~Rangwani, S.~Takemori, K.~Samanta, Y.~Umeda, and V.~B. Radhakrishnan.
\newblock Selective mixup fine-tuning for optimizing non-decomposable objectives.
\newblock In \emph{International Conference on Learning Representations}, 2024.

\bibitem[Ramaswamy et~al.(2015)Ramaswamy, Narasimhan, and Agarwal]{ramaswamy2015consistent}
H.~G. Ramaswamy, H.~Narasimhan, and S.~Agarwal.
\newblock Consistent classification algorithms for multi-class non-decomposable performance metrics.
\newblock \emph{arXiv preprint arXiv:1501.00287}, 2015.

\bibitem[Reid and Williamson(2009)]{reid2009surrogate}
M.~D. Reid and R.~C. Williamson.
\newblock Surrogate regret bounds for proper losses.
\newblock In \emph{International Conference on Machine Learning}, pages 897--904, 2009.

\bibitem[Ridnik et~al.(2021)Ridnik, Ben-Baruch, Zamir, Noy, Friedman, Protter, and Zelnik-Manor]{ridnik2021asymmetric}
T.~Ridnik, E.~Ben-Baruch, N.~Zamir, A.~Noy, I.~Friedman, M.~Protter, and L.~Zelnik-Manor.
\newblock Asymmetric loss for multi-label classification.
\newblock In \emph{Proceedings of the IEEE/CVF International Conference on Computer Vision}, pages 82--91, 2021.

\bibitem[Sanh et~al.(2019)Sanh, Debut, Chaumond, and Wolf]{sanh2019distilbert}
V.~Sanh, L.~Debut, J.~Chaumond, and T.~Wolf.
\newblock Distilbert, a distilled version of bert: smaller, faster, cheaper and lighter.
\newblock \emph{arXiv preprint arXiv:1910.01108}, 2019.

\bibitem[Sanyal et~al.(2018)Sanyal, Kumar, Kar, Chawla, and Sebastiani]{sanyal2018optimizing}
A.~Sanyal, P.~Kumar, P.~Kar, S.~Chawla, and F.~Sebastiani.
\newblock Optimizing non-decomposable measures with deep networks.
\newblock \emph{Machine Learning}, 107:\penalty0 1597--1620, 2018.

\bibitem[Schultheis et~al.(2024)Schultheis, Kotlowski, Wydmuch, Babbar, Borman, and Dembczynski]{schultheisconsistent}
E.~Schultheis, W.~Kotlowski, M.~Wydmuch, R.~Babbar, S.~Borman, and K.~Dembczynski.
\newblock Consistent algorithms for multi-label classification with macro-at-$ k $ metrics.
\newblock In \emph{International Conference on Learning Representations}, 2024.

\bibitem[Sokolova and Lapalme(2009)]{sokolova2009systematic}
M.~Sokolova and G.~Lapalme.
\newblock A systematic analysis of performance measures for classification tasks.
\newblock \emph{Information Processing \& Management}, 45\penalty0 (4):\penalty0 427--437, 2009.

\bibitem[Steinwart(2007)]{steinwart2007compare}
I.~Steinwart.
\newblock How to compare different loss functions and their risks.
\newblock \emph{Constructive Approximation}, 26\penalty0 (2):\penalty0 225--287, 2007.

\bibitem[Tarekegn et~al.(2024)Tarekegn, Ullah, and Cheikh]{nega2024deep}
A.~N. Tarekegn, M.~Ullah, and F.~A. Cheikh.
\newblock Deep learning for multi-label learning: A comprehensive survey.
\newblock \emph{arXiv preprint arXiv:2401.16549}, 2024.

\bibitem[Tavker et~al.(2020)Tavker, Ramaswamy, and Narasimhan]{tavker2020consistent}
S.~K. Tavker, H.~G. Ramaswamy, and H.~Narasimhan.
\newblock Consistent plug-in classifiers for complex objectives and constraints.
\newblock In \emph{Advances in Neural Information Processing Systems}, pages 20366--20377, 2020.

\bibitem[Tewari and Bartlett(2007)]{tewari2007consistency}
A.~Tewari and P.~L. Bartlett.
\newblock On the consistency of multiclass classification methods.
\newblock \emph{Journal of Machine Learning Research}, 8\penalty0 (36):\penalty0 1007--1025, 2007.

\bibitem[Tsoumakas et~al.(2011)Tsoumakas, Spyromitros-Xioufis, Vilcek, and Vlahavas]{mulan}
G.~Tsoumakas, E.~Spyromitros-Xioufis, J.~Vilcek, and I.~Vlahavas.
\newblock Mulan: A java library for multi-label learning.
\newblock \emph{Journal of Machine Learning Research}, 12:\penalty0 2411--2414, 2011.

\bibitem[Verhulst(1838)]{Verhulst1838}
P.~F. Verhulst.
\newblock Notice sur la loi que la population suit dans son accroissement.
\newblock \emph{Correspondance Math\'ematique et Physique}, 10:\penalty0 113--121, 1838.

\bibitem[Verhulst(1845)]{Verhulst1845}
P.~F. Verhulst.
\newblock Recherches math\'ematiques sur la loi d'accroissement de la population.
\newblock \emph{Nouveaux M\'emoires de l'Acad\'emie Royale des Sciences et Belles-Lettres de Bruxelles}, 18:\penalty0 1--42, 1845.

\bibitem[Waegeman et~al.(2014)Waegeman, Dembczy{\'n}ski, Jachnik, Cheng, and H{\"u}llermeier]{waegeman2014bayes}
W.~Waegeman, K.~Dembczy{\'n}ski, A.~Jachnik, W.~Cheng, and E.~H{\"u}llermeier.
\newblock On the bayes-optimality of f-measure maximizers.
\newblock \emph{Journal of Machine Learning Research}, 15:\penalty0 3333--3388, 2014.

\bibitem[Wu and Zhu(2020)]{wu2020multi}
G.~Wu and J.~Zhu.
\newblock Multi-label classification: do hamming loss and subset accuracy really conflict with each other?
\newblock In \emph{Advances in Neural Information Processing Systems}, pages 3130--3140, 2020.

\bibitem[Wu et~al.(2021)Wu, Li, Xu, and Zhu]{wu2021rethinking}
G.~Wu, C.~Li, K.~Xu, and J.~Zhu.
\newblock Rethinking and reweighting the univariate losses for multi-label ranking: Consistency and generalization.
\newblock In \emph{Advances in Neural Information Processing Systems}, pages 14332--14344, 2021.

\bibitem[Wu et~al.(2023)Wu, Li, and Yin]{wu2023towards}
G.~Wu, C.~Li, and Y.~Yin.
\newblock Towards understanding generalization of macro-auc in multi-label learning.
\newblock In \emph{International Conference on Machine Learning}, pages 37540--37570, 2023.

\bibitem[Wu and Zhou(2017)]{wu2017unified}
X.-Z. Wu and Z.-H. Zhou.
\newblock A unified view of multi-label performance measures.
\newblock In \emph{International Conference on Machine Learning}, pages 3780--3788, 2017.

\bibitem[Wydmuch et~al.(2018)Wydmuch, Jasinska, Kuznetsov, Busa-Fekete, and Dembczynski]{wydmuch2018no}
M.~Wydmuch, K.~Jasinska, M.~Kuznetsov, R.~Busa-Fekete, and K.~Dembczynski.
\newblock A no-regret generalization of hierarchical softmax to extreme multi-label classification.
\newblock In \emph{Advances in Neural Information Processing Systems}, 2018.

\bibitem[Yan et~al.(2018)Yan, Koyejo, Zhong, and Ravikumar]{yan2018binary}
B.~Yan, S.~Koyejo, K.~Zhong, and P.~Ravikumar.
\newblock Binary classification with karmic, threshold-quasi-concave metrics.
\newblock In \emph{International Conference on Machine Learning}, pages 5531--5540, 2018.

\bibitem[Ye et~al.(2012)Ye, Chai, Lee, and Chieu]{ye2012optimizing}
N.~Ye, K.~M.~A. Chai, W.~S. Lee, and H.~L. Chieu.
\newblock Optimizing {F}-measures: a tale of two approaches.
\newblock In \emph{International Conference on Machine Learning}, pages 1555--1562, 2012.

\bibitem[Yu et~al.(2014)Yu, Jain, Kar, and Dhillon]{yu2014large}
H.-F. Yu, P.~Jain, P.~Kar, and I.~Dhillon.
\newblock Large-scale multi-label learning with missing labels.
\newblock In \emph{International Conference on Machine Learning}, pages 593--601, 2014.

\bibitem[Yu and Blaschko(2015)]{yu2015learning}
J.~Yu and M.~Blaschko.
\newblock Learning submodular losses with the lov{\'a}sz hinge.
\newblock In \emph{International Conference on Machine Learning}, pages 1623--1631, 2015.

\bibitem[Zhang et~al.(2020)Zhang, Ramaswamy, and Agarwal]{zhang2020convex}
M.~Zhang, H.~G. Ramaswamy, and S.~Agarwal.
\newblock Convex calibrated surrogates for the multi-label f-measure.
\newblock In \emph{International Conference on Machine Learning}, pages 11246--11255, 2020.

\bibitem[Zhang and Zhou(2013)]{zhang2013review}
M.-L. Zhang and Z.-H. Zhou.
\newblock A review on multi-label learning algorithms.
\newblock \emph{IEEE Transactions on Knowledge and Data Engineering}, 26\penalty0 (8):\penalty0 1819--1837, 2013.

\bibitem[Zhang(2004{\natexlab{a}})]{Zhang2003}
T.~Zhang.
\newblock Statistical behavior and consistency of classification methods based on convex risk minimization.
\newblock \emph{The Annals of Statistics}, 32\penalty0 (1):\penalty0 56--85, 2004{\natexlab{a}}.

\bibitem[Zhang(2004{\natexlab{b}})]{zhang2004statistical}
T.~Zhang.
\newblock Statistical analysis of some multi-category large margin classification methods.
\newblock \emph{Journal of Machine Learning Research}, 5\penalty0 (Oct):\penalty0 1225--1251, 2004{\natexlab{b}}.

\bibitem[Zhang et~al.(2018)Zhang, Liu, Zhou, and Yang]{zhang2018faster}
X.~Zhang, M.~Liu, X.~Zhou, and T.~Yang.
\newblock Faster online learning of optimal threshold for consistent f-measure optimization.
\newblock In \emph{Advances in Neural Information Processing Systems}, 2018.

\bibitem[Zhang and Zhang(2024{\natexlab{a}})]{zhang2024generalization}
Y.~Zhang and M.-L. Zhang.
\newblock Generalization analysis for multi-label learning.
\newblock In \emph{International Conference on Machine Learning}, 2024{\natexlab{a}}.

\bibitem[Zhang and Zhang(2024{\natexlab{b}})]{zhang2024label}
Y.-F. Zhang and M.-L. Zhang.
\newblock Generalization analysis for label-specific representation learning.
\newblock In \emph{Advances in Neural Information Processing Systems}, pages 104904--104933, 2024{\natexlab{b}}.

\bibitem[Zhang and Zhang(2025)]{zhang2025tight}
Y.-F. Zhang and M.-L. Zhang.
\newblock Tight and fast bounds for multi-label learning.
\newblock In \emph{International Conference on Machine Learning}, 2025.

\bibitem[Zhang and Sabuncu(2018)]{zhang2018generalized}
Z.~Zhang and M.~Sabuncu.
\newblock Generalized cross entropy loss for training deep neural networks with noisy labels.
\newblock In \emph{Advances in Neural Information Processing Systems}, 2018.

\bibitem[Zhong(2025)]{zhong2025fundamental}
Y.~Zhong.
\newblock \emph{Fundamental Novel Consistency Theory: {H}-Consistency Bounds}.
\newblock PhD thesis, New York University, 2025.

\end{thebibliography}
\bibliographystyle{abbrvnat}

\newpage
\appendix
\renewcommand{\contentsname}{Contents of Appendix}
\tableofcontents
\addtocontents{toc}{\protect\setcounter{tocdepth}{3}} 
\clearpage

\section{Related Work}
\label{app:related-work}

The fields of multi-label learning, the optimization of generalized performance metrics, and the study of statistical consistency have seen extensive research. This section briefly reviews the literature most pertinent to our study. For comprehensive surveys on multi-label learning, the reader is referred to \citep{bogatinovski2022comprehensive,pillai2017designing,zhang2013review} and the recent deep learning survey by \citet{nega2024deep}.

\textbf{Multi-Label learning.}  Many classification tasks require predicting multiple labels for a single instance, such as identifying objects within an image or determining multiple topics in a text document. Multi-class multi-label classification has become a central task in modern machine learning, with broad applications and growing theoretical interest \citep{mccallum1999multi}. This importance has driven both algorithmic advances \citep{deng2011fast,elisseeff2001kernel,kapoor2012multilabel,petterson2011submodular,yu2014large} and rigorous analysis, particularly in understanding \emph{Bayes-consistency} for multi-label learning \citep{cheng2010bayes,gao2011consistency,dembczynski2011exact,dembczynski2012consistent,dembczynski2013optimizing,waegeman2014bayes,koyejo2015consistent,menon2019multilabel,zhang2020convex,schultheisconsistent}.

A central challenge in this domain is designing principled algorithms that go beyond specific cases (e.g., Hamming loss) and apply to general metric classes in multi-label learning. Such metrics include averaged accuracy ($1 -$ Hamming loss) \citep{gao2011consistency}, averaged precision \citep{menon2019multilabel}, the $F$-measure and its averaged variants \citep{ye2012optimizing,dembczynski2011exact,dembczynski2013optimizing,waegeman2014bayes,zhang2020convex}, partial ranking loss \citep{gao2011consistency,dembczynski2012consistent}, the Jaccard measure \citep{sokolova2009systematic}, and others such as one-error, subset accuracy, coverage, macro/micro/instance-F1 and AUC \citep{wu2017unified}.

Two main frameworks have emerged for analyzing classifier performance with respect to such metrics: \emph{Empirical Utility Maximization (EUM)} and \emph{Decision Theoretic Analysis (DTA)} \citep{ye2012optimizing,koyejo2015consistent}. EUM evaluates metrics based on the aggregate confusion matrix, while DTA considers instance-level losses. Note that the EUM framework is also referred to as the \emph{Population Utility (PU)} framework in some literature \citep{dembczynski2017consistency}. Most prior theoretical work focuses on Bayes-optimal classifiers or Bayes-consistent algorithms under DTA for specific metrics \citep{cheng2010bayes,gao2011consistency,dembczynski2011exact,dembczynski2013optimizing,menon2019multilabel,zhang2020convex,schultheisconsistent}. Other research provides generalization guarantees for multi-label learning \citep{busa2022regret,wu2021rethinking,wu2023towards,wu2020multi,wydmuch2018no}. Recently, \citet{zhang2024generalization,zhang2024label,zhang2025tight} have derived tighter generalization bounds that are less dependent on the number of labels, provided certain mild assumptions hold.

\textbf{Optimization of generalized metrics.} While multi-label
learning presents unique challenges, the drive to optimize beyond
simple accuracy has also been a strong theme in binary and multi-class
classification. Recent advances in deep learning have led to heuristic
approaches for optimizing complex metrics, such as the sigmoidF1 loss
\citep{benedict2022sigmoidf1} and Selective Mixup for non-decomposable
objectives \citep{ramasubramanian2024selective}. While effective in
practice, these methods often lack formal consistency guarantees.

Much of the research in designing algorithms for these generalized
metrics (either specific instances or broad families) has centered on
characterizing the Bayes-optimal classifier
\citep{KoyejoNatarajanRavikumarDhillon2014,ParambathUsunierGrandvalet2014}. A
wide array of other research avenues includes extensions and analyses
of the plug-in rule
\citep{berger2020threshold,dembczynski2013optimizing,dembczynski2017consistency,
  lipton2014thresholding,narasimhan2014statistical,tavker2020consistent,yan2018binary,ye2012optimizing};
work on surrogate regret bounds
\citep{kotlowski2016surrogate,reid2009surrogate}; various optimization
strategies like online
\citep{busa2015online,kar2014online,kotlowski2024general,zhang2018faster}
and constrained optimization \citep{narasimhan2019optimizing};
structural loss optimization
\citep{bao2020calibrated,berman2018lovasz,eban2017scalable,joachims2005support,kar2014online,yu2015learning};
and extensions to multi-class classification
\citep{cheng2016efficient,fathony2020ap,luo2021minimax,narasimhan2015consistent,narasimhan2015optimizing,
  narasimhan2016optimizing,narasimhan2024consistent,natarajan2016optimal,ramaswamy2015consistent,
  sanyal2018optimizing}.

\textbf{Statistical consistency.} While many algorithms come with
Bayes-consistency guarantees, this notion applies only to the class of
all measurable functions and lacks convergence rates. Recent work has
begun to address these limitations via $\sH$-consistency
\citep{awasthi2022multi,mao2023cross,MaoMohriZhong2023ranking,MaoMohriMohriZhong2023twostage,MaoMohriZhong2023characterization,MaoMohriZhong2023structured,MaoMohriZhong2023rankingabs,MaoMohriZhong2024deferral,MaoMohriZhong2024predictor,MaoMohriZhong2024score,mao2024h,mao2024realizable,mao2024regression,MohriAndorChoiCollinsMaoZhong2024learning,cortes2024cardinality,cortes2025balancing,cortes2025improved,mao2025enhanced,MaoMohriZhong2025mastering,mao2025theory,zhong2025fundamental,desalvo2025budgeted,CortesMaoMohriZhong2026defid,CortesMohriZhong2026mod,mohri2025beyond,MohriZhong2026rllm,MohriZhong2026slin}. In the context of multi-label
learning, \citet{mao2024multi} recently established $\sH$-consistency
bounds for the DTA framework. Our work complements this by
establishing $\sH$-consistency for the EUM framework.

\citet{koyejo2015consistent} provides the most general prior analysis of EUM-consistency. However, they rely on a two-stage plug-in algorithm based on Bayes-optimal thresholding. As shown by \citet{mao2025principled} in binary settings, such thresholding approaches can be suboptimal for restricted hypothesis classes. Our work addresses this by designing surrogate losses that allow for direct joint optimization of the metric, supported by non-asymptotic $\sH$-consistency bounds.

\section{Exact \texorpdfstring{$\mathcal{O}(\num)$}{O(l)} Factorization of the Surrogate Loss}
\label{app:factorization}

As described in Section~\ref{sec:efficient} of the main text, a direct evaluation of the combinatorial summation over $2^\num$ label configurations in Equation~\eqref{eq:sur} appears computationally prohibitive. Here, we formally demonstrate how it perfectly factorizes into an exact $\mathcal{O}(\num)$ computation, mirroring the efficiency of standard Binary Cross Entropy.

Recall that since our generalized target metrics decompose linearly over labels (Eq.~\eqref{eq:target-equiv-2}), the shifted cost term is additive across labels:
$$ \ov{\sfL}_{\text{sum}} - \ov{\sfL}(y', y) = \sum_{k=1}^\num c_k(y'_k, y_k), \quad \text{where } c_k(y'_k, y_k) = \frac{\ov{\sfL}_{\text{sum}}}{\num} - \sfL_{\bgamma,k}(y'_k, y_k). $$
The inner exponential sum factorizes exactly because the exponential of a sum is the product of exponentials:
$$ u(y') = \sum_{y'' \in \sY} \exp\left(\sum_{i=1}^\num (y''_i - y'_i)h(x, i)\right) = \prod_{i=1}^\num Z_i(y'_i), \quad \text{where } Z_i(y'_i) = \sum_{y''_i \in \{-1,1\}} \exp((y''_i - y'_i)h(x, i)). $$

\textbf{1. Factorization for the $\tau = 0$ (Logistic-type) Surrogate:} \\
Applying the logarithm converts the inner product into a sum: $\log u(y') = \sum_{i=1}^\num \log Z_i(y'_i)$. The total surrogate loss is the outer summation over $\sY$:
$$ \sum_{y' \in \sY} \left( \sum_{k=1}^\num c_k(y'_k, y_k) \right) \left( \sum_{i=1}^\num \log Z_i(y'_i) \right) $$
Because we are summing over the independent binary space $\sY = \{-1, 1\}^\num$, we can distribute the product and swap the summations. 
For cross terms ($k \neq i$), marginalizing out the $\num-2$ other labels yields a constant factor of $2^{\num-2}$. For matching terms ($k = i$), marginalizing out the $\num-1$ other labels yields $2^{\num-1}$.
To compute this in exactly $\mathcal{O}(\num)$ time, we precompute three 1D marginals for each label $j \in \{1..\num\}$ in $\mathcal{O}(1)$ time:
$A_j = \sum_{y'_j \in \{-1,1\}} c_j(y'_j, y_j)$, \quad $B_j = \sum_{y'_j \in \{-1,1\}} \log Z_j(y'_j)$, \quad $C_j = \sum_{y'_j \in \{-1,1\}} c_j(y'_j, y_j) \log Z_j(y'_j)$.
Using these precomputed arrays, the entire $2^\num$ outer summation collapses to an exact closed-form evaluation:
$$ \text{Total Loss} = 2^{\num-2} \left[ \left( \sum_{k=1}^\num A_k \right) \left( \sum_{i=1}^\num B_i \right) - \sum_{j=1}^\num A_j B_j \right] + 2^{\num-1} \sum_{j=1}^\num C_j. $$

\textbf{2. Factorization for the $\tau > 0$ Surrogate:} \\
The function $\Phi_\tau$ raises $u(y')$ to $-\tau$. This preserves the product structure: $u(y')^{-\tau} = \prod_{i=1}^\num W_i(y'_i)$, where $W_i(y'_i) = Z_i(y'_i)^{-\tau}$.
The outer summation for the first term becomes:
$$ \sum_{y' \in \sY} \left( \sum_{k=1}^\num c_k(y'_k, y_k) \right) \prod_{i=1}^\num W_i(y'_i) $$
By distributing the sum over $k$ and pushing the marginal summations inside the independent variables, this perfectly factorizes into a sum of $\num$ independent terms:
$$ = \sum_{k=1}^\num \left[ \left( \sum_{y'_k \in \{-1,1\}} c_k(y'_k, y_k) W_k(y'_k) \right) \prod_{i \neq k} \left( \sum_{y'_i \in \{-1,1\}} W_i(y'_i) \right) \right] $$
By precomputing the label-wise marginal sums $S_i = W_i(1) + W_i(-1)$ and their total product $P = \prod_{i=1}^\num S_i$ in $\mathcal{O}(\num)$ time, the $k$-th term of the outer sum simplifies to $\frac{P}{S_k} \sum_{y'_k \in \{-1,1\}} c_k(y'_k, y_k) W_k(y'_k)$. Evaluating and summing these $\num$ terms requires exactly $\mathcal{O}(\num)$ operations. 

Thus, for any $\tau \ge 0$, the $2^\num$ combinatorial sum strictly factorizes without any approximation, matching the efficiency of standard BCE.

\section{Deferred Proofs for Section~\ref{sec:formulation} (Problem Formulation)}
\label{app:proofs-sec3}

\subsection{Proof of Theorem~\ref{thm:equiv}}
\label{app:equiv}

\Equiv*
\begin{proof}
Assume that there exists a hypothesis $h^* \in \sH$ such that
$\cL(h^*) = \cL^*(\sH)$ holds. We
have for all $h \in \sH$:
\begin{align*}
  \lambda^*
  = \cL(h^*) & =  \frac{\E
    \bracket*{ \ell_{\balpha}(h^*, x, y) }}{\E
    \bracket*{ \ell_{\bbeta}(h^*, x, y) }}
  \geq \frac{\E
    \bracket*{ \ell_{\balpha}(h, x, y) }}
  {\E \bracket*{ \ell_{\bbeta}(h, x, y) }}.
\end{align*}
Thus, under the assumption, the following holds for all $h \in \sH$:
\begin{align*}
  \lambda^* \E_{(x, y) \sim \sD} \bracket*{ \ell_{\bbeta}(h, x, y) }
  & \geq \E_{(x, y) \sim \sD} \bracket*{ \ell_{\balpha}(h, x, y) },\\
  \lambda^* \E_{(x, y) \sim \sD} \bracket*{ \ell_{\bbeta}(h^*, x, y) }
  & = \E_{(x, y) \sim \sD} \bracket*{ \ell_{\balpha}(h^*, x, y) }.
\end{align*}
This implies that $\sE_{\ell^{\lambda^*\!\!}}(h^*)  = \sE_{\ell^{\lambda^*\!\!}}(\sH) = 0$, which completes one direction of the proof.

Assume now that there exists $h^* \in \sH$ such that
$\sE_{\ell^{\lambda^*\!\!}}(h^*) =
\sE_{\ell^{\lambda^*\!\!}}(\sH) = 0$ holds. We have for
all $h \in \sH$:
\begin{align*}
& \E_{(x, y) \sim \sD} \bracket*{\ell_{\balpha}(h^*, x, y)} - \lambda^* \E_{(x, y) \sim \sD} \bracket*{ \ell_{\bbeta}(h^*, x, y)} = 0\\
& \lambda^* \E_{(x, y) \sim \sD} \bracket*{ \ell_{\bbeta}(h, x, y)}  - \E_{(x, y) \sim \sD} \bracket*{\ell_{\balpha}(h, x, y)}\geq 0.
\end{align*}
Thus, we have $\lambda^* = \frac{\E_{(x, y) \sim
    \sD} \bracket*{ \ell_{\balpha}(h^*, x, y) }}{\E_{(x, y) \sim \sD}
  \bracket*{ \ell_{\bbeta}(h^*, x, y) }} = \cL(h^*)
\geq \cL(h)$ for all $h \in \sH$, which implies 
$\cL(h^*) = \cL^*(\sH)$. This
completes the proof.
\end{proof}

\subsection{Proof of Theorem~\ref{thm:equiv-non}}
\label{app:equiv-non}
\EquivNon*
\begin{proof}
Since we have $\E_{(x, y) \sim \sD} \bracket*{ \ell_{\bbeta}(h, x, y) } > 0$, the following equivalence holds for all $h \in \sH$ and $\eta \geq 0$:
\begin{align*}
\sE_{\ell^{\lambda^*\!\!}}(h) \leq \eta &\iff \lambda^* \E_{(x, y) \sim \sD} \bracket*{ \ell_{\bbeta}(h, x, y)} - \E_{(x, y) \sim \sD} \bracket*{\ell_{\balpha}(h, x, y)}  \leq \eta \tag{def. of $\ell^{\lambda^*\!\!}$}\\
&\iff \lambda^* \leq \frac{\E_{(x, y) \sim \sD} \bracket*{ \ell_{\balpha}(h, x, y) }}{\E_{(x, y) \sim \sD} \bracket*{ \ell_{\bbeta}(h, x, y) }} + \frac{\eta}{\E_{(x, y) \sim \sD} \bracket*{ \ell_{\bbeta}(h, x, y) }} \tag{$\E_{(x, y) \sim \sD} \bracket*{ \ell_{\bbeta}(h, x, y) } > 0$}\\
&\iff \cL^*(\sH) - \cL(h) \leq \frac{\eta}{\E_{(x, y) \sim \sD} \bracket*{ \ell_{\bbeta}(h, x, y) }} \tag{$\lambda^* = \cL^*(\sH)$}.
\end{align*}
This completes the proof.
\end{proof}

\section{Deferred Proofs for Section~\ref{sec:mll} (\texorpdfstring{$\sH$}{H}-Consistency)}
\label{app:proof-bound-comp}

Before proceeding with the proof of Theorem~\ref{Thm:bound_comp}, we first introduce some notation and
definitions. Given a cost-sensitive surrogate loss function $\sur$ and a
hypothesis set $\sH$, we denote the conditional error by
$\sC_{\sur}(h, x) = \E_{y \mid x} \bracket*{\sur(h, x, y)}$, the
best-in-class conditional error by $\sC^*_{\sur}\paren*{\sH, x} =
\inf_{ h \in \sH} \sC_{\sur}(h, x)$, and the conditional regret by
$\Delta \sC_{\sur, \sH}\paren*{h, x} = \sC_{\sur}(h, x) -
\sC^*_{\sur}\paren*{\sH, x}$. We then present a general theorem, which
shows that to derive $\sH$-consistency bounds in multi-label learning
with a concave function $\Gamma$, it is only necessary to upper bound
the conditional regret of the target cost-sensitive loss by that of the
surrogate loss with the same $\Gamma$.

\begin{theorem}
\label{Thm:tool-Gamma}
Let $\sfL$ be a cost-sensitive loss and $\sur$ be a surrogate loss.
Given a concave function $\Gamma \colon \Rset_{+} \to \Rset_{+}$. If
the following condition holds for all $h \in \sH$ and $x \in \sX$:
\begin{equation}
\label{eq:c-regret-Gamma}
 \Delta \sC_{\sfL, \sH}(h, x) \leq \Gamma
\paren*{\Delta \sC_{\sur, \sH}(h, x)},
\end{equation}
then, for any distribution and for all hypotheses $h \in \sH$,
    \begin{equation}
     \sE_{\sfL}(h)- \sE^*_{\sfL}(\sH) + \sM_{\sfL}(\sH) \leq \Gamma \paren*{\sE_{\sur}(h) - \sE^*_{\sur}(\sH) + \sM_{\sur}(\sH)}.
    \end{equation}
\end{theorem}
\begin{proof}
By the definitions, the expectation of the conditional regrets for $\sfL$ and $\sur$ can be expressed as:
\begin{align*}
\E_{x} \bracket*{\Delta \sC_{\sfL, \sH}(h, x)} & = \sE_{\sfL}(h) - \sE^*_{\sfL}(\sH) + \sM_{\sfL}(\sH)\\
\E_{x} \bracket*{\Delta \sC_{\sur, \sH}(h, x)} & = \sE_{\sur}(h) - \sE^*_{\sur}(\sH) + \sM_{\sur}(\sH).
\end{align*}
Thus, by taking the expectation on both sides of \eqref{eq:c-regret-Gamma} and using Jensen's inequality, we have
\begin{align*}
\sE_{\sfL}(h) - \sE^*_{\sfL}(\sH) + \sM_{\sfL}(\sH) 
& = \E_{x} \bracket*{\Delta \sC_{\sfL, \sH}(h, x)}\\
& \leq \E_{x} \bracket*{ \Gamma
\paren*{\Delta \sC_{\sur, \sH}(h, x)} } \tag{Eq. \eqref{eq:c-regret-Gamma}}\\
& \leq 
\Gamma \paren*{\E_{x} \bracket*{\Delta \sC_{\sur, \sH}(h, x)}} \tag{concavity of $\Gamma$}\\
& = \Gamma \paren*{\sE_{\sur}(h) - \sE^*_{\sur}(\sH) + \sM_{\sur}(\sH)}.
\end{align*}
This completes the proof.
\end{proof}
To derive $\sH$-consistency bounds using Theorem~\ref{Thm:tool-Gamma},
we will characterize the conditional regret of a cost-sensitive loss
$\sfL$. For simplicity, we first introduce some notation. For any $x
\in \sX$, let $\yy(x) = \argmin_{y' \in \sY} \E_{y \mid x}
\bracket*{\ov \sfL(y', y)} \in \sY$, where we choose the label with
the lowest index under the natural ordering of labels as the
tie-breaking strategy.  To simplify the notation further, we will drop
the dependency on $x$. Specifically, we use $\yy$ to denote $\yy(x)$
and $\hh$ to denote $\hh(x)$.  Additionally, we define $\ch = \E_{y
  \mid x} \bracket*{\ov \sfL(\hh, y)} $, $\cy = \E_{y \mid x}
\bracket*{\ov \sfL(\yy, y)}$ and $\cyy = \E_{y \mid x} \bracket*{\ov
  \sfL(y', y)}$, $\forall y' \in \sY$.
\begin{lemma}
\label{lemma:delta_target}
Let $\sH = \sF^{\num}$. Assume that $\sF$ is complete. Then, the conditional regret of a cost-sensitive target loss $\sfL$ can be expressed as follows:
$
\Delta \sC_{\sfL, \sH}(h, x) = \ch - \cy.
$
\end{lemma}

\begin{proof}
By definition, the conditional error of $\sfL$ can be expressed as follows: 
\begin{equation}
\sC_{\sfL}(h, x) = \E_{y \mid x} \bracket*{\sfL(h, x, y)} = \E_{y \mid x} \bracket*{\ov \sfL(\hh(x), y)} = \ch.
\end{equation}
Since $\sH = \sF^{\num}$ and $\sF$ is complete, for any $x \in \sX$, $\curl*{\hh(x) \colon h \in \sH} = \sY$. Then, the best-in-class conditional error of $\sfL$ can be expressed as follows: 
\begin{equation}
\sC^*_{\sfL}\paren*{\sH, x} = \inf_{ h \in \sH} \sC_{\sfL}(h, x) = \inf_{ h \in \sH} \E_{y \mid x} \bracket*{\ov \sfL(\hh(x), y)} = \E_{y \mid x} \bracket*{\ov \sfL(\yy(x), y)} = \cy.
\end{equation}
Therefore, $\Delta \sC_{\sfL, \sH}(h, x) = \sC_{\sfL}(h, x) - \sC^*_{\sfL}\paren*{\sH, x} = \ch - \cy.$
\end{proof}
Next, by using Lemma~\ref{lemma:delta_target}, we will upper bound the
conditional regret of the target cost-sensitive loss $\sfL$ by that of
the surrogate loss $\sur$ with a concave function $\Gamma$.

\subsection{Proof of Theorem~\ref{Thm:bound_comp}}
\label{app:bound_comp-sum}

\BoundComp*
\begin{proof}
Recall that we adopt the following notation: $\ch = \E_{y \mid x} \bracket*{\ov \sfL(\hh, y)} $, $\cy = \E_{y \mid x} \bracket*{\ov \sfL(\yy, y)}$ and $\cyy = \E_{y \mid x} \bracket*{\ov \sfL(y', y)}$, $\forall y' \in \sY$. Let $\sum_{y', y \in \sY} \ov\sfL(y', y) = \ov\sfL_{\rm{sum}}$. We will denote by $\s(h, x, y') = \frac{e^{\sum_{i = 1}^\num y'_i h(x, i)}}{ \sum_{y'' \in \sY} e^{\sum_{i = 1}^\num y''_i h(x, i)} }$ and simplify notation by using $\syy$, thereby dropping the dependency on $h$ and $x$. It is clear that $\syy \in [0, 1]$. Next, we will analyze case by case.

\textbf{The case where $\tau = 0$}: the conditional error of $\sfL_{\tau}$ can be expressed as follows:
\begin{align*}
\sC_{\sfL_{\tau}}(h, x) 
& = \E_{y \mid x} \bracket*{\sum_{y' \in \sY} \paren*{ \sum_{y', y \in \sY} \ov\sfL(y', y) - \ov \sfL(y', y) } \log \paren*{\sum_{y'' \in \sY} e^{\sum_{i = 1}^\num \paren*{ y''_i - y'_i} h(x, i)}}}\\
& = -\sum_{y' \in \sY} (\ov\sfL_{\rm{sum}}  -  \cyy) \log(\syy)
\end{align*}
For any $\hh \neq \yy$, we define $\s^{\mu}$ as follows: set $\s^{\mu}_{y'} = \syy$ for all $y' \neq \yy$ and $y' \neq \hh$; define $\s^{\mu}_{\hh} = \sy - \mu$; and let $\s^{\mu}_{\yy} = \sh + \mu$. Note that $\s^{\mu}$ can be realized by some $h^{\mu} \in \sH$ under the assumption.  Then, leveraging $\ch \geq \cy$, we have
\begin{align*}
& \Delta \sC_{\sfL_{\tau}, \sH}(h, x)\\
& \geq \paren*{-\sum_{y' \in \sY} (\ov\sfL_{\rm{sum}}  -  \cyy) \log(\syy)} - \inf_{\mu \in \Rset} \paren*{-\sum_{y' \in \sY} (\ov\sfL_{\rm{sum}}  -  \cyy) \log\paren*{\s^{\mu}_{y'}}}\\
& = \sup_{\mu \in \Rset} \curl*{ (\ov\sfL_{\rm{sum}} - \ch) \bracket*{ \log(\sy - \mu) - \log(\sh)} + (\ov\sfL_{\rm{sum}} - \cy) \bracket*{ \log(\sh + \mu) - \log(\sy) } }\\
& = (\ov\sfL_{\rm{sum}} - \ch) \log \frac{\paren*{\sh + \sy}(\ov\sfL_{\rm{sum}} - \ch)}{\sy \paren*{2\ov\sfL_{\rm{sum}} - \cy - \ch}} + (\ov\sfL_{\rm{sum}} - \cy) \log \frac{\paren*{\sh + \sy}(\ov\sfL_{\rm{sum}} - \cy)}{\sh \paren*{2\ov\sfL_{\rm{sum}} - \cy - \ch}}
\tag{supremum is attained when $\mu^*  = \frac{(\ov\sfL_{\rm{sum}} - \cy) \sy - (\ov\sfL_{\rm{sum}} - \ch) \sh}{2\ov\sfL_{\rm{sum}} - \cy - \ch}$}\\
& \geq (\ov\sfL_{\rm{sum}} - \ch) \log \frac{2(\ov\sfL_{\rm{sum}} - \ch)}{\paren*{2\ov\sfL_{\rm{sum}} - \ch - \cy}} + (\ov\sfL_{\rm{sum}} - \cy) \log \frac{2(\ov\sfL_{\rm{sum}} - \cy)}{\paren*{2\ov\sfL_{\rm{sum}} - \ch - \cy}}\\
\tag{minimum is attained when $\sh  = \sy$}\\
& \geq \frac{(\ch - \cy)^2}{2 \paren*{2\ov\sfL_{\rm{sum}} - \ch - \cy}}
\tag{$a\log \frac{2a}{a + b} + b\log \frac{2b}{a + b}\geq \frac{(a - b)^2}{2(a + b)}, \forall a, b \in[0, 1]$}\\
& \geq \frac{(\ch - \cy)^2}{4 \ov \sfL_{\rm{sum}}}.
\end{align*}
Therefore, by Lemma~\ref{lemma:delta_target}, $\Delta \sC_{\sfL, \sH}(h, x) \leq 2 \paren*{\ov \sfL_{\rm{sum}}}^{\frac12} \paren*{\Delta \sC_{\sfL_{\tau}, \sH}(h, x)}^{\frac12}$. By Theorem~\ref{Thm:tool-Gamma}, we complete the proof.

\textbf{The case where $\tau \in (0, 1)$}: The conditional error of $\sfL_{\tau}$ can be expressed as:
\begin{align*}
\sC_{\sfL_{\tau}}(h, x) = \frac{1}{\tau} \sum_{y' \in \sY} (\ov\sfL_{\rm{sum}}  -  \cyy) \paren*{1 - (\syy)^\tau}.
\end{align*}
For any $\hh \neq \yy$, we define $\s^{\mu}$ similarly as before.  Then, we have
\begin{align*}
& \Delta \sC_{\sfL_{\tau}, \sH}(h, x)\\
& \geq \frac{1}{\tau} \sum_{y' \in \sY} (\ov\sfL_{\rm{sum}}  -  \cyy) \paren*{1 - (\syy)^\tau} - \inf_{\mu \in \Rset} \paren*{\frac{1}{\tau} \sum_{y' \in \sY} (\ov\sfL_{\rm{sum}}  -  \cyy) \paren*{1 - (\s^{\mu}_{y'})^\tau} }\\
& = \frac{1}{\tau} \sup_{\mu \in \Rset} \curl*{ (\ov\sfL_{\rm{sum}} - \ch) \bracket*{-(\sh)^{\tau} + (\sy - \mu)^{\tau} } + (\ov\sfL_{\rm{sum}} - \cy) \bracket*{ -(\sy)^{\tau} + (\sh + \mu)^{\tau} } }\\
&   \geq  \frac{1}{\tau}\paren*{\sh+ \sy}^{\tau}\paren*{(\ov\sfL_{\rm{sum}} - \ch)^{\frac{1}{1 - \tau}} + (\ov\sfL_{\rm{sum}} - \cy)^{\frac{1}{1 - \tau}}}^{1 - \tau} - \frac{1}{\tau}(\ov\sfL_{\rm{sum}} - \cy)\sy^{\tau} - \frac{1}{\tau}(\ov\sfL_{\rm{sum}} - \ch)\sh^{\tau}
\tag{evaluating at $\mu^*  =   \frac{(\ov\sfL_{\rm{sum}} - \cy)^{\frac{1}{1 - \tau}}\sy - (\ov\sfL_{\rm{sum}} - \ch)^{\frac{1}{1 - \tau}}\sh}{(\ov\sfL_{\rm{sum}} - \ch)^{\frac{1}{1 - \tau}}+(\ov\sfL_{\rm{sum}} - \cy)^{\frac{1}{1 - \tau}}}$}
\\
& \geq \frac{1}{\tau n^{\tau}} \bracket*{2^{\tau} \paren*{(\ov\sfL_{\rm{sum}} - \ch)^{\frac{1}{1 - \tau}} + (\ov\sfL_{\rm{sum}} - \cy)^{\frac{1}{1 - \tau}}}^{1 - \tau} - (\ov\sfL_{\rm{sum}} - \ch)-(\ov\sfL_{\rm{sum}} - \cy)}
\tag{minimum is attained when $\sh  = \sy  =  \frac{1}{n}$}\\
& \geq \frac{(\ch - \cy)^2}{4 \ov \sfL_{\rm{sum}} n^{\tau}}
\tag{$2\paren*{\frac{a^{\frac{1}{1 - \tau}} + b^{\frac{1}{1 - \tau}}}{2}}^{1 - \tau} - (a + b)  \geq \frac{\tau}{2}(a - b)^2$}.
\end{align*}
Therefore, by Lemma~\ref{lemma:delta_target}, $\Delta \sC_{\sfL, \sH}(h, x) \leq 2 \paren*{\ov \sfL_{\rm{sum}}}^{\frac12} n^{\frac{\tau}{2}}\paren*{\Delta \sC_{\sfL_{\tau}, \sH}(h, x)}^{\frac12}$. By Theorem~\ref{Thm:tool-Gamma}, we complete the proof.

\textbf{The case where $\tau \geq 1$}: The conditional error of $\sfL_{\tau}$ can be expressed as:
\begin{align*}
\sC_{\sfL_{\tau}}(h, x) =  \sum_{y' \in \sY} (\ov\sfL_{\rm{sum}}  -  \cyy) \paren*{1 - (\syy)^\tau}.
\end{align*}
For any $\hh \neq \yy$, we define $\s^{\mu}$ as before. Then, we have
\begin{align*}
& \Delta \sC_{\sfL_{\tau}, \sH}(h, x)\\
& \geq \frac{1}{\tau} \sum_{y' \in \sY} (\ov\sfL_{\rm{sum}}  -  \cyy) \paren*{1 - (\syy)^\tau} - \inf_{\mu \in \Rset} \paren*{\frac{1}{\tau} \sum_{y' \in \sY} (\ov\sfL_{\rm{sum}}  -  \cyy) \paren*{1 - (\s^{\mu}_{y'})^\tau} }\\
& = \frac{1}{\tau} \sup_{\mu \in \Rset} \curl*{ (\ov\sfL_{\rm{sum}} - \ch) \bracket*{-(\sh)^{\tau} + (\sy - \mu)^{\tau} } + (\ov\sfL_{\rm{sum}} - \cy) \bracket*{ -(\sy)^{\tau} + (\sh + \mu)^{\tau} } }\\
& = \frac{1}{\tau} \curl*{ -(\ov\sfL_{\rm{sum}} - \ch)(\sh)^{\tau} - (\ov\sfL_{\rm{sum}} - \cy)(\sy)^{\tau} + (\ov\sfL_{\rm{sum}} - \cy)(\sh + \sy)^{\tau} }\tag{evaluating at $\mu^*  = \sy$}
\\
&   \geq \frac{1}{\tau} \paren*{\sh}^{\tau} (\ch - \cy)
\tag{$(a + b)^{\tau} \geq a^{\tau} + b^{\tau}$ for $\tau \geq 1$ and using $\ch \geq \cy$}
\\
& \geq \frac1{\tau n^{\tau}} (\ch - \cy)
\tag{minimum is attained when $\sh = \frac{1}{n}$}.
\end{align*}
Therefore, by Lemma~\ref{lemma:delta_target}, $\Delta \sC_{\sfL, \sH}(h, x) \leq \tau n^{\tau} \Delta \sC_{\sfL_{\tau}, \sH}(h, x)$. By Theorem~\ref{Thm:tool-Gamma}, we complete the proof.
\end{proof}

\section{Algorithms and Theoretical Guarantees for Section~\ref{sec:algo}}
\label{app:algo}

This appendix elaborates on the algorithms designed for optimizing multi-label generalized metrics introduced in Section~\ref{sec:algo}. We begin by characterizing the optimal parameter $\lambda^*$, which forms the basis for an initial binary search algorithm that assumes oracle access to the sign of certain expected losses (Appendix~\ref{app:algo-oracle}). We then transition to practical algorithms that use empirical minimization of a surrogate loss $\sfL_\tau$ (for the target loss $\ell^\lambda$) and introduce strategies, including binary search (Appendix~\ref{app:algo-surrogate}) and cross-validation (Appendix~\ref{app:algo-cv}), for selecting an effective $\lambda$. The rigorous theoretical guarantees underpinning these algorithms are also provided here.

As discussed in Section~\ref{sec:algo}, a key challenge is that while costs $\ov \sfL(y', y)$ in a general cost-sensitive learning problem are typically known \emph{a priori}, the costs embedded within our target loss $\ell^{\lambda^*\!\!}$ depend on $\lambda^*$, which is unknown. Our approach, therefore, involves methods to determine or approximate $\lambda^*$. 

Recall that $\lambda^* = \cL^*(\sH) = \sup_{h \in \sH} \frac{\E_{(x, y) \sim \sD} \bracket*{ \ell_{\balpha}(h, x, y) }}{\E_{(x, y) \sim \sD} \bracket*{ \ell_{\bbeta}(h, x, y) }}$. The expected loss of a hypothesis $h \in \sH$ with respect to the loss function $\ell^{\lambda}$ is:
\begin{equation*}
  \sE_{\ell^\lambda}(h) = \lambda \E_{(x, y) \sim \sD} \bracket*{ \ell_{\bbeta}(h, x, y)} - \E_{(x, y) \sim \sD} \bracket*{\ell_{\balpha}(h, x, y)}.
\end{equation*}

\subsection{Characterization of the Optimal Parameter}
\label{app:algo-lambda}

The proof of Theorem~\ref{thm:lambda-star} from Section~\ref{sec:algo} is as follows.

\LambdaStar*
\begin{proof}
  For any $h \in \sH$, define $f(h)$ and $g(h)$ to simplify notation:
  \begin{align*}
    f(h)
    = \frac{\E_{(x, y) \sim \sD} \bracket*{\ell_{\balpha}(h, x, y)}}{\E_{(x, y) \sim \sD} \bracket*{\ell_{\bbeta}(h, x, y)}},
\quad g(h)
= \E_{(x, y) \sim \sD} \bracket*{\ell_{\bbeta}(h, x, y)}.
\end{align*}
  By assumption, we have $g(h) > 0$ for all $h \in \sH$ and $\lambda^* = \sup_{h \in \sH} f(h)$ and $\sE_{\ell^{\lambda^*\!\!}}^*(\sH) = \inf_{h \in \sH} \curl*{(\lambda^* - f(h)) g(h)}$. Since $g$ is upper-bounded by $\ol_\bbeta$ and $(\lambda^* - f(h)) \geq 0$, it follows that
  \[
  \sE_{\ell^{\lambda^*\!\!}}^*(\sH) \leq \inf_{h \in \sH} \curl*{\lambda^* - f(h)} \, \ol_\bbeta = 0.
  \]
By definition of $\sE_{\ell^{\lambda^*\!\!}}^*(\sH)$ as an infimum, for any $\eta > 0$, there exists $h_\eta \in \sH$ such that
\[
\sE_{\ell^{\lambda^*\!\!}}^*(\sH) + \eta > (\lambda^* - f(h_\eta)) g(h_\eta) \geq 0.
\]
Since $\sE_{\ell^{\lambda^*\!\!}}^*(\sH) + \eta > 0$ for all $\eta > 0$, it follows that $\sE_{\ell^{\lambda^*\!\!}}^*(\sH) \geq 0$.  Combining the two inequalities yields $\sE_{\ell^{\lambda^*\!\!}}^*(\sH) = 0$. This establishes the first equality.

Next, assume $\lambda - \lambda^* > 0$.  By the definition of $\lambda^*$ as a supremum, we have $\lambda - f(h) > 0$ for all $h \in \sH$. This implies $\sE_{\ell^\lambda}^*(\sH) = \inf_{h \in \sH} \curl*{(\lambda - f(h)) g(h)} \geq \ul_\bbeta \inf_{h \in \sH} \curl*{(\lambda - f(h))} = \ul_\bbeta (\lambda - \lambda^*) > 0$, thus $\sE_{\ell^\lambda}^*(\sH) > 0$, which proves one direction of the statement.

Now, assume $\lambda - \lambda^* < 0$.  By the definition of $\lambda^*$ as an supremum, for any $\eta > 0$, there exists $h_\eta \in \sH$ such that $f(h_\eta) > \lambda^* - \eta$.  Choose $\eta < (\lambda^* - \lambda)$.  This implies $\sE_{\ell^{\lambda}}^*(\sH) \leq (\lambda - f(h_\eta)) g(h_\eta) \leq (\lambda + \eta - \lambda^*) g(h_\eta)$. Since $(\lambda + \eta - \lambda^*) < 0$ and $g(h_\eta) > 0$, it follows that $\sE_{\ell^\lambda}^*(\sH) < 0$. This proves the other direction.
\end{proof}

\subsection{Oracle Binary Search and Convergence Rate}
\label{app:algo-oracle}

This characterization naturally leads to a binary search algorithm for computing an $\e$-approximation of $\lambda^*$, provided one has oracle access to the sign of $\sE_{\ell^\lambda}^*(\sH)$.  The pseudocode for this oracle-based algorithm is given in Algorithm~\ref{alg:binary-search-oracle}. Here, $\lambda_{\min}$ and $\lambda_{\max}$ define the initial search interval for $\lambda$; these bounds can often be estimated from the properties of the specific generalized metric.

\begin{algorithm}[htbp]
\caption{Binary Search Estimation of $\lambda^*$ (Oracle-Based)}
\label{alg:binary-search-oracle}
\begin{algorithmic}[1]
  \REQUIRE Tolerance parameter $\e > 0$. Assumed bounds $\bracket*{\lambda_{\min}, \lambda_{\max}}$ for $\lambda^*$.
  \STATE Initialize search interval $[a, b] \gets \bracket*{\lambda_{\min}, \lambda_{\max}}$
   \REPEAT
   \STATE Set candidate $\lambda \gets \frac{a + b}{2}$
      \IF{$\sE_{\ell^\lambda}^*(\sH) > 0$}
         \STATE Update interval $[a, b] \gets [a, \lambda]$
      \ELSE
         \STATE Update interval $[a, b] \gets [\lambda, b]$
      \ENDIF
      \UNTIL{$\abs*{b - a} \leq \e$}
  \RETURN $\lambda$ (an $\e$-approximation of $\lambda^*$)
\end{algorithmic}
\end{algorithm}

The rate of convergence for this oracle-based algorithm is established by the following theorem. We first introduce a necessary lemma.

\begin{lemma}
\label{lemma:approx}
Fix $\e > 0$ and assume $\abs*{\lambda - \lambda^*} \leq \e$. Then, the following inequality holds:
\[
\sE_{\ell^\lambda}^*(\sH) \leq  \e \ol_\bbeta.
\]
\end{lemma}
\begin{proof}
By definition of $\sE_{\ell^\lambda}(h)$, the following holds for any $h \in \sH$:
\begin{align*}
\sE_{\ell^\lambda}(h) & = \sE_{\ell^{\lambda^*\!\!}}(h) + (\lambda - \lambda^*) \E_{(x, y) \sim \sD}[\ell_\bbeta(h, x, y)] \leq \sE_{\ell^{\lambda^*\!\!}}(h) + \e \ol_\bbeta.
\end{align*}
Thus, by definition of $\sE_{\ell^\lambda}^*(\sH)$ as an infimum, for any $h \in \sH$, we have $\sE_{\ell^\lambda}^*(\sH)\leq \sE_{\ell^{\lambda^*\!\!}}(h) + \e \ol_\bbeta$. Taking the infimum of the right-hand side over $\sH$ yields $\sE_{\ell^\lambda}^*(\sH) \leq \sE_{\ell^{\lambda^*\!\!}}^*(\sH) + \e \ol_\bbeta$.  By Theorem~\ref{thm:lambda-star}, we have $\sE_{\ell^{\lambda^*\!\!}}^*(\sH) = 0$. This completes the proof.
\end{proof}

\begin{theorem}[Convergence Rate of Oracle Binary Search]
\label{thm:rate}
For a fixed tolerance $\e > 0$, Algorithm~\ref{alg:binary-search-oracle} yields an estimate $\lambda$ for $\lambda^*$ within $\mathcal{O}\paren*{\log_2\paren*{\frac{\lambda_{\max} - \lambda_{\min}}{\e}}}$ iterations. This estimate $\lambda$ satisfies the property:
\[
\sE_{\ell^\lambda}^*(\sH) \leq \sE_{\ell^{\lambda^*\!\!}}^*(\sH) + \e \ol_\bbeta.
\]
Leveraging the result from Theorem~\ref{thm:lambda-star} that $\sE_{\ell^{\lambda^*\!\!}}^*(\sH) = 0$, this bound simplifies to $\sE_{\ell^\lambda}^*(\sH) \leq \e \ol_\bbeta$.
\end{theorem}
\begin{proof}
  By definition of the binary search, we have $|\lambda - \lambda^*| \leq \e$. Thus, the inequality holds by Lemma~\ref{lemma:approx}. The time complexity follows straightforwardly from the properties of binary search.
\end{proof}

\subsection{Generalization Bounds for the Surrogate}
\label{app:algo-gen}

In practice, oracle access to the sign of $\sE_{\ell^\lambda}^*(\sH)$ is unavailable. Instead, we can approximate $\sE_{\ell^\lambda}^*(\sH)$ by finding a hypothesis $\h h_S$ that minimizes an empirical surrogate loss $\sfL_{\tau}$ (for the target $\ell^\lambda$) on a labeled sample $S$ of size $m$. To bridge this gap, we first provide a generalization bound for the target loss $\ell^{\lambda}$ by leveraging the
$\sH$-consistency bounds from Section~\ref{sec:guarantees}. Let $S$ be a sample of size $m$. We denote by $\Rad_m^\lambda(\sH)$ the Rademacher complexity of the function class $\curl*{(x, y) \mapsto \sfL_{\tau}(h, x, y) \colon h
  \in \sH}$ and let $B_{\lambda} = \sup_{h, x, y }\sfL_{\tau}(h, x, y)$
be an upper bound on the surrogate loss $\sfL_{\tau}$.

\begin{theorem}[Estimation Error Bound]
\label{thm:generalization}
Suppose the surrogate loss $\sfL_{\tau}$ satisfies a $\Gamma$-$\sH$-consistency bound with respect to the target loss $\ell^{\lambda}$. Let $\h h_S \in \sH$ be an empirical minimizer of $\sfL_{\tau}$ over a sample $S$ drawn from $\sD^m$. Then, for any $\delta > 0$, with probability at least $1 - \delta$:
\begin{equation*}
\sE_{\ell^\lambda}(\h h_S) - \sE_{\ell^\lambda}^*(\sH)
\leq \Gamma
  \paren*{\sM_{\sfL_{\tau}}(\sH) + 4 \Rad_m^\lambda(\sH) +
2 B_\lambda \sqrt{\tfrac{\log \frac{2}{\delta}}{2m}}}
- \sM_{\ell^\lambda}(\sH).
\end{equation*}
\end{theorem}
\begin{proof}
We aim to relate the expected surrogate risk of the empirical minimizer, $\sE_{\sfL_{\tau}}(\h h_S)$, to the best possible expected surrogate risk, $\sE_{\sfL_{\tau}}^*(\sH)$.
Let $S$ be a sample of $m$ instances. We denote the Rademacher complexity of the function class $\curl*{(x, y) \mapsto \sfL_{\tau}(h, x, y) \colon h \in \sH}$ derived from $\sfL_\tau$ and $\sH$ by $\Rad_m^\lambda(\sH)$, and let $B_{\lambda} = \sup_{h, x, y }\sfL_{\tau}(h, x, y)$ be a uniform bound on the values of $\sfL_{\tau}$.
Standard uniform convergence theory (e.g., \citep{MohriRostamizadehTalwalkar2018}) states that for any $\delta > 0$, with at least $1-\delta$ probability over $S$, all $h \in \sH$ satisfy:
\begin{equation}
\label{eq:proof_uc_bound_alt2}
\sE_{\sfL_{\tau}}(h) \leq \h \sE_{\sfL_{\tau}, S}(h) + C'_{m,\delta} \quad \text{and} \quad \h \sE_{\sfL_{\tau}, S}(h) \leq \sE_{\sfL_{\tau}}(h) + C'_{m,\delta},
\end{equation}
where $C'_{m,\delta} = 2 \Rad_m^\lambda(\sH) + B_\lambda \sqrt{\frac{\log (2/\delta)}{2m}}$.

Let $\h h_S$ be the hypothesis that minimizes $\h \sE_{\sfL_{\tau}, S}(h)$ over $h \in \sH$.
For any chosen $\epsilon > 0$, there exists an $h^* \in \sH$ such that $\sE_{\sfL_{\tau}}(h^*) \leq \sE_{\sfL_{\tau}}^*(\sH) + \epsilon$.
We can now construct a chain of inequalities:
\begin{align*}
\sE_{\sfL_{\tau}}(\h h_S) &\leq \h \sE_{\sfL_{\tau}, S}(\h h_S) + C'_{m,\delta} && \text{(Using Eq. \eqref{eq:proof_uc_bound_alt2} for } \h h_S \text{)} \\
&\leq \h \sE_{\sfL_{\tau}, S}(h^*) + C'_{m,\delta} && \text{(By definition of } \h h_S \text{ as empirical minimizer)} \\
&\leq \left( \sE_{\sfL_{\tau}}(h^*) + C'_{m,\delta} \right) + C'_{m,\delta} && \text{(Using Eq. \eqref{eq:proof_uc_bound_alt2} for } h^* \text{)} \\
&= \sE_{\sfL_{\tau}}(h^*) + 2C'_{m,\delta} \\
&\leq \left( \sE_{\sfL_{\tau}}^*(\sH) + \epsilon \right) + 2C'_{m,\delta} && \text{(By definition of } h^* \text{)}
\end{align*}
Thus, we have $\sE_{\sfL_{\tau}}(\h h_S) - \sE_{\sfL_{\tau}}^*(\sH) \leq 2C'_{m,\delta} + \epsilon$.
Substituting $C'_{m,\delta}$:
\[ \sE_{\sfL_{\tau}}(\h h_S) - \sE_{\sfL_{\tau}}^*(\sH) \leq 2 \left(2 \Rad_m^\lambda(\sH) + B_\lambda \sqrt{\frac{\log (2/\delta)}{2m}}\right) + \epsilon. \]
As this inequality is valid for any positive $\epsilon$, we can conclude by letting $\epsilon \to 0^+$:
\[ \sE_{\sfL_{\tau}}(\h h_S) - \sE_{\sfL_{\tau}}^*(\sH) \leq 4 \Rad_m^\lambda(\sH) + 2 B_\lambda \sqrt{\frac{\log (2/\delta)}{2m}}. \]
This result provides the necessary bound on the surrogate's estimation error, which is then used in conjunction with the assumed $\Gamma$-$\sH$-consistency property to establish Theorem~\ref{thm:generalization}.
\end{proof}

While Theorem~\ref{thm:generalization} applies for a fixed $\lambda$, this bound can be extended to hold uniformly for all $\lambda \in \bracket*{\lambda_{\min}, \lambda_{\max}}$ by covering the interval with sub-intervals of size $1/m$.

\begin{lemma}
\label{lemma:approx-Rad}
Let $B_{\Phi_{\tau}}$ be an upper bound for the function $\Phi_{\tau}$. For any $\epsilon_0 > 0$, if $\abs*{\lambda_1 - \lambda_2} \leq \epsilon_0$, then the following inequalities demonstrate the Lipschitz continuity of Rademacher complexity and the maximum surrogate loss with respect to $\lambda$:
\[
\abs*{\Rad_m^{\lambda_1}(\sH) - \Rad_m^{\lambda_2}(\sH)} \leq \frac{K_{\sfL} \epsilon_0}{m} , \quad \text{and} \quad \abs*{B_{\lambda_1} - B_{\lambda_2}} \leq K_{\sfL} \epsilon_0,
\]
where $K_{\sfL} = 4 \num \max_{j, k} \abs*{\beta_k^j} B_{\Phi_{\tau}}$ is the effective Lipschitz constant.
\end{lemma}
\begin{proof}
The surrogate loss $\sfL_{\tau}^{\lambda}$ incorporates the parameter $\lambda$ through its dependence on the target loss $\ell^\lambda = \lambda \ell_\beta - \ell_\alpha$. The specific structure of $\ell^\lambda$ (related to Eq.~\eqref{eq:target-equiv-2} which defines $\ell^{\lambda^*}$) leads to the Lipschitz continuity of $\sfL_{\tau}^{\lambda}$ itself with respect to $\lambda$. One can show that for any $(h, x, y) \in \sH \times \sX \times \sY$:
\begin{align*}
\abs*{\sfL_{\tau}^{\lambda_1}(h, x, y) - \sfL_{\tau}^{\lambda_2}(h, x, y)} \leq K_{\sfL} \abs*{\lambda_1 - \lambda_2},
\end{align*}
where $K_{\sfL} = 4 \num \max_{j, k} \abs*{\beta_k^j}  B_{\Phi_{\tau}}$.
This point-wise Lipschitz property of $\sfL_{\tau}^{\lambda}$ directly implies the Lipschitz property for $B_{\lambda} = \sup_{h, x, y }\sfL_{\tau}^{\lambda}(h, x, y)$:
\begin{align*}
\abs*{B_{\lambda_1} - B_{\lambda_2}} &\leq \sup_{h,x,y} \abs*{\sfL_{\tau}^{\lambda_1}(h, x, y) - \sfL_{\tau}^{\lambda_2}(h, x, y)} \\
&\leq K_{\sfL} \abs*{\lambda_1 - \lambda_2} \leq K_{\sfL} \epsilon_0.
\end{align*}
For the Rademacher complexity, $\Rad_m^{\lambda}(\sH) = \mathbb{E}_{\boldsymbol{\sigma}} \left[ \sup_{h \in \sH} \frac{1}{m} \sum_{i=1}^m \sigma_i \sfL_{\tau}^{\lambda}(h, x_i, y_i) \right]$, the Lipschitz continuity of $\sfL_{\tau}^{\lambda}$ allows the application of standard results (e.g., Talagrand's contraction lemma). These establish that:
\[
\abs*{\Rad_m^{\lambda_1}(\sH) - \Rad_m^{\lambda_2}(\sH)} \leq \frac{K_{\sfL} \abs*{\lambda_1 - \lambda_2}}{m} \leq \frac{K_{\sfL} \epsilon_0}{m}.
\]
This completes the proof.
\end{proof}

\begin{theorem}[Uniform Estimation Error Bound]
\label{thm:generalization-uniform}
Assume that the following $\sH$-consistency bound holds for all $h\in \sH$, $\lambda \in [\lambda_{\min}, \lambda_{\max}]$, and all distributions:
\begin{equation*}
  \sE_{\ell^\lambda}(h) - \sE_{\ell^\lambda}^*(\sH)
  + \sM_{\ell^\lambda}(\sH)
  \leq \Gamma \paren*{\sE_{\sfL_{\tau}}(h) - \sE_{\sfL_{\tau}}^*(\sH)
    + \sM_{\sfL_{\tau}}(\sH)}.
\end{equation*}
Let $B_{\Phi_{\tau}}$ be an upper bound for the function $\Phi_{\tau}$.
Then, for any $\delta > 0$, with probability at least $1 - \delta$ over the draw of a sample $S$ from $\sD^m$, the following estimation bound holds for an empirical minimizer $\h h_S \in \sH$ of the surrogate loss $\sfL_{\tau}$ over $S$, uniformly for all $\lambda \in [\lambda_{\min}, \lambda_{\max}]$:
\begin{align*}
\sE_{\ell^\lambda}(\h h_S) - \sE_{\ell^\lambda}^*(\sH)
& \leq \Gamma
\paren[\bigg]{\sM_{\sfL_{\tau}}(\sH) + 4\Rad_m^{\lambda}(\sH) + \frac{8 \num \max_{j, k} \abs*{\beta_k^j} B_{\Phi_{\tau}} }{m^2} \\
& \qquad
  + \paren*{2 B_{\lambda} + \frac{4 \num \max_{j, k} \abs*{\beta_k^j}  B_{\Phi_{\tau}} }{m}}\sqrt{\tfrac{\log \frac{2 N_\lambda}{\delta}}{2m}}} + \frac{\ol_\bbeta}{m},
\end{align*}
where $N_\lambda = m(\lambda_{\max} - \lambda_{\min})$ is the number of points in the covering net for $\lambda$.
\end{theorem}
\begin{proof}
The proof uses the covering argument strategy.
We define a discrete set of points $\Lambda_{\text{net}} = \{\lambda_k\}_{k=1}^{N_\lambda}$ to cover the interval $[\lambda_{\min}, \lambda_{\max}]$. Let $\lambda_k = \lambda_{\min} + \frac{k - 1}{m}$ for $1 \leq k \leq N_\lambda$, where $N_\lambda = m (\lambda_{\max} - \lambda_{\min})$. This ensures that for any $\lambda \in [\lambda_{\min}, \lambda_{\max}]$, there exists a $\lambda_k \in \Lambda_{\text{net}}$ such that the distance $\abs*{\lambda - \lambda_k} \leq \frac{1}{2m}$. Let this distance be $\epsilon_0 = 1/(2m)$.

We apply Theorem~\ref{thm:generalization} to each $\lambda_k \in \Lambda_{\text{net}}$. To ensure this bound holds simultaneously for all $N_\lambda$ points with an overall probability of at least $1-\delta$, we use a union bound argument. This requires setting the failure probability for each individual point-wise bound to $\delta' = \delta/N_\lambda$. Consequently, the term $\log(2/\delta')$ appears in the bound for each $\lambda_k$, which is $\log(2N_\lambda/\delta)$.
Thus, with probability at least $1 - \delta$, for all $k \in \{1, \dots, N_\lambda\}$:
\begin{equation}
\label{eq:uniform_proof_bound_on_net}
\sE_{\ell^{\lambda_k}}(\h h_S) - \sE_{\ell^{\lambda_k}}^*(\sH)
\leq \Gamma
\paren[\bigg]{\sM_{\sfL_{\tau}}(\sH) + 4 \Rad_m^{\lambda_k}(\sH)
  + 2 B_{\lambda_k} \sqrt{\tfrac{\log \frac{2 N_\lambda}{\delta}}{2m}}}.
\end{equation}

Now, consider an arbitrary $\lambda \in [\lambda_{\min}, \lambda_{\max}]$ and let $\lambda_k$ be its closest point in $\Lambda_{\text{net}}$, so $\abs*{\lambda - \lambda_k} \leq \epsilon_0 = 1/(2m)$. 
The transition from $\lambda_k$ to $\lambda$ introduces differences in several terms. The target loss terms $\sE_{\ell^\lambda}(\h h_S)$ and $\sE_{\ell^\lambda}^*(\sH)$ change, which by Lemma~\ref{lemma:approx} contributes an additive term bounded by $\ol_\bbeta \abs{\lambda - \lambda_k} \leq \ol_\bbeta/(2m)$ outside the $\Gamma(\cdot)$ function.
Inside the argument of $\Gamma(\cdot)$, by Lemma~\ref{lemma:approx-Rad}:
$4 \Rad_m^{\lambda_k}(\sH) \leq 4\left(\Rad_m^{\lambda}(\sH) + \frac{K_{\sfL}}{2m^2}\right) = 4\Rad_m^{\lambda}(\sH) + \frac{2K_{\sfL}}{m^2}$.
$2 B_{\lambda_k} \leq 2\left(B_{\lambda} + \frac{K_{\sfL}}{2m}\right) = 2B_{\lambda} + \frac{K_{\sfL}}{m}$.
Plugging these into the argument of $\Gamma(\cdot)$ in Eq.~\eqref{eq:uniform_proof_bound_on_net} and adding the term $\ol_\bbeta/m$ yields the bound stated in Theorem~\ref{thm:generalization-uniform}.
\end{proof}

\subsection{Empirical Surrogate-Based Search}
\label{app:algo-surrogate}

Let $\e_m$ denote the right-hand of the uniform bound in Theorem~\ref{thm:generalization-uniform}. Building upon Algorithm~\ref{alg:binary-search-oracle}, we introduce a practical version in Algorithm~\ref{alg:binary-search-surrogate}. With high probability, the condition $\h \sE_{\ell^\lambda, S}(\h h_\lambda) > \e_m$ implies $\sE_{\ell^\lambda}^*(\sH) > 0$, and similarly, $\h \sE_{\ell^\lambda, S}(\h h_\lambda) < -\e_m$ implies $\sE_{\ell^\lambda}^*(\sH) < 0$. Thus, these empirical conditions serve as proxies for the oracle conditions.

\begin{algorithm}[htbp]
   \caption{Algorithm for Optimizing Generalized Metrics via Surrogate Minimization}
   \label{alg:binary-search-surrogate}
\begin{algorithmic}[1]
  \REQUIRE Tolerance parameters $\e > 0$, $\e_m > 0$. Assumed bounds $\bracket*{\lambda_{\min}, \lambda_{\max}}$ for $\lambda^*$.
  \STATE Initialize search interval $[a, b] \gets \bracket*{\lambda_{\min}, \lambda_{\max}}$
   \REPEAT
   \STATE  Set candidate $\lambda \gets \frac{a + b}{2}$
   \STATE  Find hypothesis $\h h_\lambda \gets \argmin_{h \in \sH} \h \sE_{\sfL_{\tau}, S}(h)$ (minimizing surrogate for $\ell^\lambda$)
      \IF{$\h \sE_{\ell^\lambda, S}(\h h_\lambda) > \e_m$}
         \STATE Update interval  $[a, b] \gets [a, \lambda]$
      \ELSIF{$\h \sE_{\ell^\lambda, S}(\h h_\lambda) < -\e_m$}
         \STATE Update interval $[a, b] \gets [\lambda, b]$
      \ELSE
         \RETURN $\h h_\lambda$
      \ENDIF
   \UNTIL{$\abs*{b - a} \leq \e$}
   \STATE \COMMENT{Interval sufficiently small, return current best based on midpoint}
   \RETURN $\h h_{\frac{a + b}{2}}$
\end{algorithmic}
\end{algorithm}

\begin{theorem}[Performance Guarantee of Empirical Surrogate-Based Search]
\label{thm:algo2}
Set  $\e = \frac{\e_m}{2 \ol_\bbeta}$. For any $\delta > 0$, Algorithm~\ref{alg:binary-search-surrogate} terminates in $\mathcal{O}\paren*{\log_2\paren*{\frac{\lambda_{\max} - \lambda_{\min}}{\e}}}$ time. The hypothesis $\h h_\lambda$ it returns is guaranteed, with probability at least $1 - \delta$, to satisfy:
\begin{equation*}
\cL(\h h_\lambda)
\geq  \cL^*(\sH) - \frac{4 \e_m}{\ul_\bbeta}.
\end{equation*}
\end{theorem}
\begin{proof}
Suppose that $\h h_\lambda$ is the solution returned at step 10 of Algorithm~\ref{alg:binary-search-surrogate}.  Then, by definition of the algorithm, we must have $|\h \sE_{\ell^{\lambda}}(\h h_\lambda)| \leq \e_m$, which, by the standard generalization bound \citep{MohriRostamizadehTalwalkar2018}, implies that (with high probability)
\[
\sE_{\ell^\lambda}(\h h_\lambda) \leq 2 \e_m.
\]
Expanding the definition of the surrogate risk, this gives
$
\lambda \E[\ell_\bbeta(h, x, y)] - \E[\ell_\balpha(h, x, y)]\leq 2 \e_m.
$
Dividing through by $\E[\ell_\bbeta(h, x, y)]$ yields
\begin{equation}
\label{eq:intermediate}
\cL(\h h_\lambda) \geq \lambda - \frac{2 \e_m}{\E[\ell_\bbeta(h, x, y)]} \geq \lambda - \frac{2 \e_m}{\ul_\bbeta}.
\end{equation}
Next, observe that for all $h \in \sH$, we have
$
\sE_{\ell^\lambda}(h) - \sE_{\ell^{\lambda^*\!\!}}(h) = (\lambda^* - \lambda) \E[\ell_\bbeta(h, x, y)].
$
Rearranging gives
$
(\lambda^* - \lambda) = \frac{\sE_{\ell^\lambda}(h) - \sE_{\ell^{\lambda^*\!\!}}(h)}{\E[\ell_\bbeta(h, x, y)]}.
$
Since $\sE_{\ell^{\lambda^*\!\!}}(\sH) = 0$, by definition of the infimum, it follows that $\sE_{\ell^{\lambda^*\!\!}}(h) \geq 0$.  In view of that, we can write:
$
(\lambda^* - \lambda) \leq \frac{\sE_{\ell^\lambda}(h)}{\E[\ell_\bbeta(h, x, y)]}.
$
Applying this inequality with $h = \h h_\lambda$ and substituting into \eqref{eq:intermediate}, we obtain  
\[
\cL(\h h_\lambda) \geq \lambda^* - \frac{2 \e_m}{\ul_\bbeta} - \frac{\sE_{\ell^\lambda}(\h h_\lambda)}{\E[\ell_\bbeta(\h h_\lambda, x, y)]} \geq \lambda^* - \frac{4 \e_m}{\ul_\bbeta}.
\]
This ends the analysis of that case.

Now, consider the case where the algorithm terminates with $|b - a| \leq \e$. In this case, we have (with high probability) $|\lambda - \lambda^*| \leq \e$. By Lemma~\ref{lemma:approx}, this implies
$
\sE^*_{\ell^{\lambda}}(\sH) \leq \e \ol_\bbeta.
$
Combining this with the estimation bound for $\ell^{\lambda}$, we have (with high probability)
$
\sE_{\ell^\lambda}(\h h_\lambda) \leq \sE^*_{\ell^{\lambda}}(\sH) + \e_m \leq \e \ol_\bbeta + \e_m.
$
Thus, expanding the expected loss:
\[
\lambda \E[\ell_{\bbeta}(\h h_\lambda, x, y)] - \E[\ell_{\balpha}(\h h_\lambda, x, y)] \leq \e \ol_\bbeta + \e_m.
\]
Dividing by $\E[\ell_{\bbeta}(\h h_\lambda, x, y)]$ on both sides yields:
\[
\lambda - \cL(\h h_\lambda) \leq \frac{\e \ol_\bbeta + \e_m}{\ul_\bbeta}.
\]
Therefore, rearranging for $\cL(\h h_\lambda)$:
\[
\cL(\h h_\lambda) \geq \lambda - \frac{\e \ol_\bbeta + \e_m}{\ul_\bbeta} \geq \lambda^* - \e - \frac{\e \ol_\bbeta + \e_m}{\ul_\bbeta}.
\]
Choosing $\e = \e_m/(2 \ol_\bbeta)$ yields
\[
\cL(\h h_\lambda) \geq \lambda^* - \frac{\e_m}{2 \ol_\bbeta} - \frac{3 \e_m}{2\ul_\bbeta} \geq \lambda^* - \frac{2\e_m}{\ul_\bbeta}.
\]
This completes the proof.
\end{proof}
This theorem indicates that when $\e_m$ is small (i.e., for a sufficiently large sample size relative to the complexity of $\sH$, and a small minimizability gap), the performance of the returned hypothesis $\h h_\lambda$ closely approaches the optimal
performance $\cL^*(\sH)$ achievable within the hypothesis class $\sH$.

\subsection{Cross-Validation Search}
\label{app:algo-cv}

In practice, the theoretical expression for $\e_m$ may involve constants that are difficult to estimate tightly, or the minimizability gap might not be accurately approximated, especially in non-realizable cases. In such scenarios, $\lambda$ can be treated as a hyperparameter and tuned using cross-validation. Algorithm~\ref{alg:binary-search-cv-formal} is the fully detailed version of Algorithm~\ref{alg:cv-main} introduced in Section~\ref{sec:algo}.

\begin{algorithm}[htbp]
   \caption{Algorithm for Optimizing Generalized Metrics using Cross-Validation (Formal)}
   \label{alg:binary-search-cv-formal}
\begin{algorithmic}[1]
\REQUIRE Step size $\e$. Assumed bounds $\bracket*{\lambda_{\min}, \lambda_{\max}}$ for $\lambda^*$.
   \STATE Initialize $[a, b] \gets \bracket*{\lambda_{\min}, \lambda_{\max}}$, $\lambda^* = \lambda_{\min}$, $i = 0$
   \STATE Initialize best metric value found: $v_{\text{best}} \gets -\infty$.
   \STATE Initialize best hypothesis: $\h h_{\text{best}} \gets \text{null}$.
   \STATE Initialize best $\lambda$: $\h \lambda_{\text{best}}\gets \lambda_{\min}$.
   \STATE Initialize iteration counter $i \gets 0$.
  \REPEAT
   \STATE Set candidate $\lambda \gets b - i \e$ \COMMENT{Linear scan from b downwards.}
    \STATE Find hypothesis $\h h_\lambda \gets \argmin_{h \in \sH} \h \sE_{\sfL_{\tau}, S}(h)$ (minimizing surrogate for $\ell^\lambda$)
    \STATE  Evaluate current hypothesis: $v_{\text{curr}} \gets \h \cL_S(\h h_\lambda)$.
\IF{$v_{\text{curr}} > v_{\text{best}}$}
\STATE $v_{\text{best}} \gets v_{\text{curr}}$.
    \STATE $\h \lambda_{\text{best}}\gets \lambda$
    \STATE $\h h_{\text{best}} \gets \h h_\lambda$
    \ENDIF
    \STATE  $i \gets i + 1$
   \UNTIL{$b - i \e < a$}
   \RETURN $\h h_{\text{best}}$
\end{algorithmic}
\end{algorithm}

\begin{theorem}[Performance Guarantee of Cross-Validation Search]
\label{thm:algo3}
Set $\e \leq \frac{\e_m}{2 \ol_\bbeta}$. For any $\delta > 0$, Algorithm~\ref{alg:binary-search-cv-formal} terminates in $\mathcal{O} \paren*{\frac{\lambda_{\max} - \lambda_{\min}}{\e}}$ time. The hypothesis $\h h_\lambda$ it returns is guaranteed, with probability at least $1 - \delta$, to satisfy:
\[
\cL(\h h_\lambda) \geq \cL^*(\sH) - \frac{2\e_m}{\ul_\bbeta}.
\]
\end{theorem}
\begin{proof}
Since the algorithm terminates with $a + i \e > b$, we have $|\lambda - \lambda^*| \leq \e$. By Lemma~\ref{lemma:approx}, this implies
$
\sE^*_{\ell^{\lambda}}(\sH) \leq \e \ol_\bbeta.
$
Combining this with the estimation bound for $\ell^{\lambda}$, we have (with high probability)
$
\sE_{\ell^\lambda}(\h h_\lambda) \leq \sE^*_{\ell^{\lambda}}(\sH) + \e_m \leq \e \ol_\bbeta + \e_m.
$
Thus, expanding the expected loss:
\[
\lambda \E[\ell_{\bbeta}(\h h_\lambda, x, y)] - \E[\ell_{\balpha}(\h h_\lambda, x, y)] \leq \e \ol_\bbeta + \e_m.
\]
Dividing by $\E[\ell_{\bbeta}(\h h_\lambda, x, y)]$ on both sides yields:
\[
\lambda - \cL(\h h_\lambda) \leq \frac{\e \ol_\bbeta + \e_m}{\ul_\bbeta}.
\]
Therefore, rearranging for $\cL(\h h_\lambda)$:
\[
\cL(\h h_\lambda) \geq \lambda - \frac{\e \ol_\bbeta + \e_m}{\ul_\bbeta} \geq \lambda^* - \e - \frac{\e \ol_\bbeta + \e_m}{\ul_\bbeta}.
\]
Choosing $\e \leq \e_m/(2 \ol_\bbeta)$ yields
\[
\cL(\h h_\lambda) \geq \lambda^* - \frac{\e_m}{2 \ol_\bbeta} - \frac{3 \e_m}{2\ul_\bbeta} \geq \lambda^* - \frac{2\e_m}{\ul_\bbeta}.
\]
This completes the proof.
\end{proof}

\section{Linear Model Benchmark Details for Section~\ref{sec:exp-linear}}
\label{app:linear-benchmarks}

This section provides the complete experimental setup and implementation details for the standard linear model benchmarks presented in Section~\ref{sec:exp-linear} of the main text.

Following the exact experimental settings used in \citet[Section~5.2]{koyejo2015consistent}, we evaluated our approach on four benchmark multi-label datasets: \texttt{scene}, \texttt{birds}, \texttt{emotions}, and \texttt{cal500}, all sourced from the Mulan project library \citep{mulan}. 

Our experiments adopted a separate linear model for each label $k$. Each model was trained using the Adam optimizer \citep{kingma2014adam}, with a learning rate of $1\times 10^{-3}$, a batch size of $128$, and a weight decay of $1\times 10^{-5}$. For the baselines, we adopted the exact same setup as \citet{koyejo2015consistent}, where logistic regression was used to estimate $\Pr(y_k = +1 \mid x)$ independently for each label $k$. Consistent with their setting, we restrict training and evaluation to labels that are associated with at least $20$ instances in each dataset, for all methods.

Model performance was evaluated using micro-averaged, macro-averaged, and instance-averaged versions of the empirical generalized metric $\h \cL_{S}$. Specifically, we considered:
\begin{itemize}
    \item The \textbf{$F_{1}$ measure} \citep{ye2012optimizing}, where the coefficients are $\balpha = \paren*{\frac12, \frac12, \frac12, \frac12}^{\num}$ and $\bbeta = \paren*{0, \frac12, \frac{1}{2}, 1}^{\num}$. 
    \item The \textbf{Jaccard measure} \citep{sokolova2009systematic}, where $\balpha = \paren*{\frac{1}{4}, \frac{1}{4}, \frac{1}{4}, \frac{1}{4}}^{\num}$ and $\bbeta = \paren*{ \frac{1}{4}, \frac{1}{4}, -\frac{1}{4}, \frac{3}{4}}^{\num}$. 
\end{itemize}
All reported metrics were averaged over ten randomly sampled cross-validation splits with standard deviations reported.

We compared our \MMO\ algorithm (specifically Algorithm~\ref{alg:binary-search-cv-formal} from Appendix~\ref{app:algo-cv}, which uses cross-validation) against three baselines from \citet{koyejo2015consistent}: Binary Relevance (BR), their Algorithm 1, and Macro-Thresholding (Macro-Thres). For the \MMO\ algorithm, we used the surrogate loss $\sfL_{\tau}$ with the parameter $\tau = 0$, which corresponds to a logistic-type loss. The performance figures for baseline methods are adopted from \citet{koyejo2015consistent}, which were consistent with our own re-evaluations where applicable. The hyperparameter $\lambda$ for our \MMO\ algorithm was selected using the exact same cross-validation procedure used for threshold selection in the baselines of \citet{koyejo2015consistent}.

Table~\ref{tab:comparison-linear} in Section~\ref{sec:exp-linear} presents the comprehensive comparison of Micro-, Macro-, and Instance-averaged $F_1$ and Jaccard scores on the Mulan datasets. As shown and discussed in Section~\ref{sec:exp-linear}, our \MMO\ algorithm consistently outperforms the three baseline methods across all datasets, averaging methods, and evaluated metrics.
\end{document}